\documentclass[journal]{IEEEtran}

\usepackage{plainsty}
\usepackage[utf8]{inputenc} 
\usepackage[T1]{fontenc}    
\usepackage{url}            
\usepackage{booktabs}       
\usepackage{amsfonts}       
\usepackage{nicefrac}       
\usepackage{microtype}      
\usepackage{xcolor}
\usepackage{bm}
\usepackage{multirow}
\usepackage{algorithm}
\usepackage{algorithmic}
\usepackage{balance}
\usepackage{cite}

\def\x{\bm{x}} 
\def\y{y} 
\def\w{\bm{w}}
 
\def\W{\bm{W}}

\def\ind{\mathbbm{1}} 
\def\E{\mathbb{E}} 
\def\P{\mathbb{P}} 
\def\T{{\top }} 
\def\S{\mathcal{S}}
\def\N{\mathcal{N}}

\def\R{\mathbb{R}}

\def\supp{\textrm{supp}}

\newcommand{\defn}{\coloneqq}
\DeclarePairedDelimiterX{\ip}[2]{\langle}{\rangle}{#1,#2}

\renewcommand{\hat}{\widehat}

\newcommand{\zero}{\textbf{0}}
\newcommand{\Ind}{\bm I}

\newcommand{\jie}[1]{#1}

\begin{document}

\title{Provable Identifiability of \jie{Two-Layer} ReLU Neural Networks via LASSO Regularization}

\author{Gen~Li, Ganghua~Wang, Jie~Ding
\thanks{G.~Li is with the Department of Statistics and Data Science, University of Pennsylvania Wharton School. G.~Wang and J.~Ding are with the School of Statistics, University of Minnesota.}
}

\markboth{IEEE Transactions on Information Theory 
}
{Shell \MakeLowercase{\textit{et al.}}: IEEE Transactions on Information Theory }

\maketitle


\begin{abstract}

LASSO regularization is a popular regression tool to enhance the prediction accuracy of statistical models by performing variable selection through the $\ell_1$ penalty, initially formulated for the linear model and its variants. In this paper, the territory of LASSO is extended to \jie{two-layer} ReLU neural networks, a fashionable and powerful nonlinear regression model. Specifically, given a neural network whose output $y$ depends only on a small subset of input $\boldsymbol{x}$, denoted by $\mathcal{S}^{\star}$, we prove that the LASSO estimator can stably reconstruct the neural network and identify $\mathcal{S}^{\star}$ when the number of samples scales logarithmically with the input dimension. This challenging regime has been well understood for linear models while barely studied for neural networks. Our theory lies in an extended Restricted Isometry Property (RIP)-based analysis framework for \jie{two-layer} ReLU neural networks, which may be of independent interest to other LASSO or neural network settings. Based on the result, we advocate a neural network-based variable selection method. Experiments on simulated and real-world datasets show promising performance of the variable selection approach compared with existing techniques.



\end{abstract}


\section{Introduction} \label{sec_intro}
    
    Given $n$ observations $(y_i,\bm \x_i)$, $i=1,\ldots,n$, we often model them with the regression form of $y_i = f(\x_i) + \xi_i$, with an unknown function $f$, $\bm \x_i \in \mathbb{R}^{p}$ being the input variables, and $\xi_i$ representing statistical errors.
    A general goal is to estimate a regression function $\hat{f}_n$ close to $f$ for prediction or interpretation. 
    This is a challenging problem when the input dimension $p$ is comparable or even much larger than the data size $n$. For linear regressions, namely $ f(\bm \x) =  \bm \w^\T \bm \x $, the least absolute shrinkage and selection operator (LASSO)~\cite{tibshirani1996regression} regularization has been established as a standard tool to estimate $f$.
    The LASSO has also been successfully used and studied in many nonlinear models such as generalized linear models~\cite{van2008high}, proportional hazards models~\cite{tibshirani1997lasso}, and neural networks~\cite{goodfellow2016deep}. \jie{For LASSO-regularized neural networks, existing works have studied different properties, such as convergence of training~\cite{wang2017convergence}, model pruning~\cite{yang2022theoretical,PQIndexPruning}, and feature selection~\cite{scardapane2017group,lemhadri2021lassonet}.
    The LASSO regularization has also been added into the standard deep learning toolbox of many open-source libraries, e.g., Tensorflow~\cite{abadi2016tensorflow} and Pytorch~\cite{paszke2019pytorch}.}
    
    \jie{Despite the practical success of LASSO in improving the generalizability and sparsification of neural networks, whether one can use LASSO for identifying significant variables is underexplored.} 
    For linear models, the variable selection problem is also known as support recovery or feature selection in different literature.
    Selection consistency requires that the probability of $\supp(\widehat{\bm \w}) = \supp(\bm \w)$ converges to one as $n \to \infty$.  
    The standard approach to selecting a parsimonious sub-model is to either solve a penalized regression problem or iteratively pick up significant variables~\cite{DingOverview}. The existing methods differ in how they incorporate unique domain knowledge (e.g., sparsity, multicollinearity, group behavior) or what desired properties (e.g., consistency in coefficient estimation, consistency in variable selection) to achieve~\cite{zhang2023information}. 
    For instance, consistency of the LASSO method~\cite{tibshirani1996regression} in estimating the significant variables has been extensively studied under various technical conditions, including sparsity, mutual coherence~\cite{donoho2001uncertainty}, restricted isometry~\cite{candes2005decoding},  irrepresentable condition~\cite{zhao2006model}, and restricted eigenvalue~\cite{bickel2009simultaneous}.

    Many theoretical studies of neural networks have focused on the generalizability.
    For example, a universal approximation theorem was established that shows any continuous multivariate function can be represented precisely by a polynomial-sized two-layer network~\cite{kolmogorov1957representation}.
    It was later shown that any continuous function could be approximated arbitrarily well by a two-layer perceptron with sigmoid activation functions~\cite{cybenko1989approximations}, and an approximation error bound of using two-layer neural networks to fit arbitrary smooth functions has been established~\cite{barron1993universal,barron1994approximation}.
    Statistically, generalization error bounds for two-layer neural networks~\cite{barron1994approximation,NN1Regularization} and multi-layer networks~\cite{neyshabur2015norm,golowich2017size,li2021rate} have been developed.
    From an optimization perspective, the parameter estimation of neural networks was cast into a tensor decomposition problem where a provably global optimum can be obtained~\cite{janzamin2015beating,ge2017learning,mondelli2018connection}. 
    Very recently, a dimension-free Rademacher complexity to bound the generalization error for deep ReLU neural networks was developed to avoid the curse of dimensionality~\cite{barron2019complexity}. It was proved that certain deep neural networks with few nonzero network parameters could achieve minimax rates of convergence~\cite{schmidt2017nonparametric}. A tight error bound free from the input dimension was developed by assuming that the data is generated from a generalized hierarchical interaction model~\cite{bauer2019deep}.  
    
    This work theoretically studies the identifiability of neural networks and uses it for variable selection. Specifically, suppose data observations are generated from a neural network with only a few nonzero coefficients. The identifiability concerns the possibility of identifying those coefficients, which may be further used to identify a sparse set of input variables that are genuinely relevant to the response. \jie{In this direction, LASSO and its variant Group LASSO~\cite{yuan2006model} have been advocated to regularize neural-network for variable selection in practice~(see, e.g.,~\cite{scardapane2017group,dinh2020consistent,lemhadri2021lassonet,yang2022theoretical}). 
    }
    
    In this paper, we consider the following class of two-layer ReLU neural networks for regression.
    \begin{align*}
     \mathcal{F}_r = \biggl\{
     & f : \bm x \mapsto f(\bm \x) = \sum_{j = 1}^r a_j\mathrm{relu}(\bm \w_j^{\T} \bm \x + b_j) , \\
     & \textrm{ where }
      a_j, b_j \in \R, \bm w_j \in \R^p \biggr\}.
    \end{align*}
    Here, $p$ and $r$ denote the input dimension and the number of neurons, respectively. We will assume that data are generated from a regression function in $\mathcal{F}_r$ perturbed by a small term.
    \jie{We will study the following two questions.}
        
    First, if the underlying regression function $f$ admits a parsimonious representation so that only a small set of input variables, $\S^\star$, is relevant, can we identify them with high probability given possibly noisy measurements $(y_i, \bm \x_i)$, for $i=1,\ldots,n$? Second, is such an $\S^\star$ estimable, meaning that it can be solved from an optimization problem with high probability, even in small-$n$ and large-$p$ regimes? 
    
    
    To address the above questions, we will establish a theory for neural networks with the LASSO regularization by considering the problem $\min_{\bm W, \bm a, \bm b} \|\bm W\|_1$ under the constraint of 
$$\frac{1}{n}\sum_{i = 1}^n \biggl(y_i - \sum_{j = 1}^r a_j\mathrm{relu}(\bm w_j^{\top}\bm x_i + b_j)\biggr)^2 \le \sigma^2,$$
    which is an alternative version of the $\ell_1$-regularization.
More notational details will be introduced in Subsection~\ref{subsec_form}.

\jie{Our theory gives positive answers to the above questions.} We theoretically show that the LASSO-type estimator can stably identify ReLU neural networks with sparse input signals, up to a permutation of hidden neurons. \jie{We only focus on the varying $n$ and $p$ and implicitly assume that the sparsity of $\bm W^{\star}$ and the number of neurons $r$ are fixed. While this does not address wide neural networks, we think it still corresponds to a nontrivial and interesting function class. For example, the class contains linear functions when input signals are bounded.}
Our result is rather general as it applies to noisy observations of $y$ and dimension regimes where the sample size $n$ is much smaller than the number of input variables $p$. The theory was derived based on new concentration bounds and function analysis that may be interesting in their own right. 
    
Inspired by the developed theory, we also propose a neural network-based variable selection method. The idea is to use the neural system as a vehicle to model nonlinearity and extract significant variables. 
Through various experimental studies, we show encouraging performance of the technique in identifying a sparse set of significant variables from large dimensional data, even if the underlying data are not generated from a neural network. Compared with popular approaches based on tree ensembles and linear-LASSO, the developed method is suitable for variable selection from nonlinear, large-dimensional, and low-noise systems. 
    
The rest of the paper is outlined as follows.
    \Autoref{sec_main} introduces the main theoretical result and proposes an algorithm to perform variable selection.
    \Autoref{sec_exp} uses simulated and real-world datasets to demonstrate the proposed theory and algorithm.
    \Autoref{sec_con} concludes the paper. 
\section{Main result}\label{sec_main}

    \subsection{Notation}
    
    Let $\bm u_\S$ denote the vector whose entries indexed in the set $\S$ remain the same as those in $\bm u$, and the remaining entries are zero. 
    For any matrix $\W \in \R^{p \times r}$, we define 
    $$
    \norm{\W}_1 = \sum_{1 \leq k \leq p , 1\leq j \leq r} \abs{w_{kj}}, \, \norm{\W}_{\mathrm{F}} = \biggl(\sum_{1 \leq k \leq p , 1\leq j \leq r} w_{kj}^2 \biggr)^{1/2}.
    $$
    Similar notations apply to vectors. The inner product of two vectors is denoted as $\langle \bm u, \bm v \rangle$. Let $\bm \w_j$ denote the $j$-th column of $\bm \W$. 
   The sparsity of a matrix $\bm \W$ refers to the number of nonzero entries in $\bm \W$.
   Let $\mathcal{N}(\bm 0, \Ind_p)$ denote the standard $p$-dimensional Gaussian distribution, and $\ind(\cdot)$ denote the indicator function. The rectified linear unit (ReLU) function is defined by $\mathrm{relu}(v)=\max\{v,0\}$ for all $v \in \R$.   

    \subsection{Formulation} \label{subsec_form}
    
    Consider $n$ independently and identically distributed (i.i.d.) observations $\{\bm x_i, y_i\}_{1\le i \le n}$ satisfying
\begin{align}
 y_i = \sum_{j = 1}^r a_j^{\star} \cdot \mathrm{relu}(\bm w_j^{\star\top }\bm x_i + b_j^{\star}) +  \xi_i 
 \label{eq_form}
\end{align}
with $\bm x_i \sim \mathcal{N}(\bm 0, \Ind_p)$,
where $r$ is the number of neurons, $a_j^{\star} \in \{1,-1\}$, $\bm w_j^{\star} \in \R^p$, $b_j^{\star} \in \R$, and $\xi_i$ denotes the random noise or approximation error obeying 
\begin{equation}
\begin{gathered}\label{eq_100}
  \frac{1}{n}\sum_{i = 1}^n \xi_i^2 \le \sigma^2. 
\end{gathered}
\end{equation}
In the above formulation, the assumption $a_j^{\star} \in \{1,-1\}$ does not lose generality since $a \cdot \mathrm{relu}(b) = ac \cdot \mathrm{relu}(b/c)$ for any $c>0$. The setting of Inequality~\ref{eq_100} is for simplicity. If $\xi_i$'s are unbounded random variables, the theoretical result to be introduced will still hold, with more explanations in the Appendix. 
The $\xi_i$'s are not necessarily i.i.d., and $\sigma$ is allowed to be zero, which reduces to the noiseless scenario.

Let $\bm \W^{\star} = [\bm w_1^{\star}, \ldots, \bm w_r^{\star}] \in \R^{p\times r}$ denote the data-generating coefficients. 
The main problem to address is whether one can stably identify those nonzero elements, given that most entries in $\bm \W^{\star} $ are zero. 
The study of neural networks from an identifiability perspective is essential. Unlike the generalizability problem that studies the predictive performance of machine learning models, the identifiability may be used to interpret modeling results and help scientists make trustworthy decisions.  
To illustrate this point, we will propose to use neural networks for variable selection in Subsection~\ref{sec_alg}. 
 
        
To answer the above questions, we propose to study the following LASSO-type optimization.
 Let $\big(\widehat{\bm W}, \widehat{\bm a}, \widehat{\bm b}\big)$ be a solution to the following optimization problem,
\begin{align}
&\min_{\bm W, \bm a, \bm b} \|\bm W\|_1\qquad  \label{eq:l1} \\
&\text{subject to } \frac{1}{n}\sum_{i = 1}^n \biggl(y_i - \sum_{j = 1}^r a_j\cdot \mathrm{relu}(\bm w_j^{\top }\bm x_i + b_j)\biggr)^2 \le \sigma^2, \nonumber 
\end{align}
within the feasible range $\bm a \in \{1,-1\}^r$, $\bm W \in \R^{p \times r}$, and $\bm b \in \R^r$.
Intuitively, the optimization operates under the constraint that the training error is not too large, and the objective function tends to sparsify $\W$. Under some regularity conditions, we will prove that the solution is indeed sparse and close to the data-generating process.

    \subsection{Main result}
    
     We make the following technical assumption.
\begin{assumption} \label{assumption2}
For some constant $B \geq 1$,  
\begin{align}
1 \le \|\bm w_j^{\star}\|_2 \le B\quad\text{and}\quad|b^{\star}_j| \le B\quad \forall 1\le j \le r. \label{assumption_a}
\end{align}
In addition, for some constant $\omega > 0$,
\begin{align}
\max_{j, k = 1,\ldots,r, j \ne k}\frac{\left|\langle \bm w_j^{\star}, \bm w_k^{\star}\rangle\right|}{\|\bm w_j^{\star}\|_2\|\bm w_k^{\star}\|_2} \le \frac{1}{r^{\omega}}. \label{assumption_b} 
\end{align}
\end{assumption}

\begin{remark}[Discussion of Assumption~\ref{assumption2}] \label{remark_assum1}
    The condition in \ref{assumption_a} is a normalization only for technical convenience, since we can re-scale $\bm w_j, b_j, y_i, \sigma$ proportionally without loss of generality. Though this condition implicitly requires  $\bm w_j^{\star} \neq \bm 0$ for all $j=1,\ldots,r$, it is reasonable since it means the neuron $j$ is not used/activated.
    The condition in \ref{assumption_b} requires that the angle of any two different coefficient vectors is not too small. This condition is analogous to a bounded-eigenvalue condition often assumed for linear regression problems, where each $w_j^{\star}$ is understood as a column in the design matrix. This condition is by no means mild or easy to verify in practice. Nevertheless, as our focused regime is large $p,n$ but small $r$, we think the condition in \ref{assumption_b} is still reasonable. For example, when $r=2$, this condition simply requires $w_1^{\star} \neq w_2^{\star}$.
\end{remark}

        
        Our main result shows that if $\bm W^{\star}$ is sparse, one can stably reconstruct a neural network when the number of samples ($n$) scales logarithmically with the input dimension ($p$).         
        A skeptical reader may ask how the constants exactly depend on the sparsity and $r$. We will provide a more elaborated result in Subsection~\ref{subsec_elab} and introduce the proof there.

    \begin{theorem} \label{thm_main}
    Under Assumption~\ref{assumption2}, there exist some constants $c_1, c_2, c_3 > 0$ depending only (polynomially) on the sparsity of $\bm \W^{\star}$ such that for any $\delta > 0$, one has with probability at least $1 - \delta$, 
    \begin{align}
    \widehat{\bm a} = \bm\Pi \bm a^{\star}\, \text{and} \,\, \|\widehat{\bm W} - \bm W^{\star}\bm\Pi^{\top}  \|_{\mathrm{F}} + \|\widehat{\bm b} - \bm\Pi \bm b^{\star}\|_2 \le c_1\sigma \label{eq_202}
    \end{align}
    for some permutation matrix $\bm\Pi$, provided that
    \begin{align}
    n > c_2\log^4\frac{p}{\delta}\qquad\text{and}\qquad \sigma < c_3.
    \end{align}
    \end{theorem}

        \begin{remark}[Interpretation of Theorem~\ref{thm_main}] \label{remark1}
            The permutation matrix $\bm\Pi$ is necessary since the considered neural networks produce identical predictive distributions (of $\y$ conditional $\bm \x$) under any permutation of the hidden neurons. The result states that the underlying neural coefficients can be stably estimated even when the sample size $n$ is much smaller than the number of variables $p$. Also, the estimation error bound is at the order of $\sigma$, the specified noise level in \ref{eq_100}. 
            
            Suppose that we define the signal-to-noise ratio (SNR) to be 
            $\E\|\bm x\|^2 / \sigma^2$.
            An alternative way to interpret the theorem is that a large SNR ensures the global minimizer to be close to the ground truth with high probability.
            One may wonder what if the $\sigma < c_3$ condition is not met. We note that if $\sigma$ is too large, the error bound in \ref{eq_202} would be loose, and it is not of much interest anyway. In other words, if the SNR is small, we may not be able to estimate parameters stably. This point will be demonstrated by experimental studies in Section~\ref{sec_exp}. 
        
        \end{remark}
        
        The estimation results in Theorem~\ref{thm_main} can be translated into variable selection results as shown in the following Corollary~\ref{coro_selection}. The connection is based on the fact that if $i$-th variable is redundant, the underlying coefficients associated with it should be zero. Let $\bm \w_{i,\cdot}^{\star}$ denote the $i$-th row of $\bm W^{\star}$. Then, 
        $$\S^{\star} = \{1 \leq i \leq p: \|\bm \w_{i,\cdot}^{\star}\|_2 > 0\}$$ characterizes the ``significant variables.'' Corollary~\ref{coro_selection} states that the set of variables with non-vanished coefficient estimates contains all the significant variables. The corollary also shows that with a suitable shrinkage of the coefficient estimates, one can achieve variable selection consistency.   
        
        
        \jie{
        \begin{corollary}[Variable selection] \label{coro_selection}
            Let $\widehat{\S}_{c_1 \sigma} \subseteq \{1,\ldots,p\}$ denote the sets of $i$'s such that $\|\widehat{\bm w}_{i,\cdot} \|_2 > c_1 \sigma$.
            Under the same assumption as in Theorem~\ref{thm_main}, and $\min_{1 \leq i\leq r} \norm{\w_{i,\cdot}^{\star}}_2>2c_1\sigma$, for any $\delta > 0$, one has  
            \begin{align}
                 \P(\S^{\star} = \widehat{\S}_{c_1 \sigma} ) \geq 1 - \delta, \nonumber
            \end{align}
             provided that $n > c_2\log^4\frac{p}{\delta}$ and $\sigma < c_3$.
        \end{corollary}
        }
        
        Considering the noiseless scenario $\sigma=0$, Theorem~\ref{thm_main} also implies the following corollary.

    \begin{corollary}[Unique parsimonious representation] \label{coro_represent}
            Under the same assumption as in Theorem~\ref{thm_main}, there exists a constant $c_2 > 0$ depending only on the sparsity of $\bm \W^{\star}$ such that for any $\delta > 0$, one has with probability at least $1 - \delta$, 
        \begin{align}
            \widehat{\bm a} = \bm\Pi \bm a^{\star}, 
            \quad 
            \widehat{\bm W} = \bm W^{\star}\bm\Pi^\T , 
            \quad\text{and}\quad 
            \widehat{\bm b} = \bm\Pi \bm b^{\star} \nonumber
        \end{align}
        for some permutation matrix $\bm\Pi$, provided that $n > c_2\log^4\frac{p}{\delta}$. 
    
    \end{corollary}
        
        Corollary~\ref{coro_represent} states that among all the possible representations $\bm W$ in \ref{eq_form} (with $\xi_i=0$), the one(s) with the smallest $\ell_1$-norm must be identical to $\bm W^{\star}$ up to a column permutation with high probability. In other words, the most parsimonious representation (in the sense of $\ell_1$ norm) of two-layer ReLU neural networks is unique. 
        This observation addresses the questions raised in Section~\ref{sec_intro}. 
        
        \jie{It is worth mentioning that since the weight matrix $\bm W$ of the neural network is row-sparse, Group-LASSO is a suitable alternative to LASSO. We leave the analysis of Group-LASSO for future study.}

\subsection{Elaboration on the main result} \label{subsec_elab}




Suppose that $\bm W^{\star}$ has at most $s$ nonzero entries. 
The following theorem is a more elaborated version of Theorem~\ref{thm_main}.

\begin{theorem} \label{thm_main_elab}
There exist some constants $c_1, c_2, c_3 > 0$ such that for any $\delta > 0$, one has with probability at least $1 - \delta$, 
\begin{align}
\widehat{\bm a} = \bm\Pi \bm a^{\star} \, \text{and} \,\, \|\widehat{\bm W} - \bm W^{\star}\bm\Pi^{\top}\|_{\mathrm{F}} + \|\widehat{\bm b} - \bm\Pi \bm b^{\star}\|_2 \le c_1\sigma
\end{align}
for some permutation $\bm\Pi \in \{0, 1\}^{r \times r}$, provided that Assumption~\ref{assumption2} holds and
\begin{align} \label{eq:samples-general}
n > c_2s^3r^{13}\log^4\frac{p}{\delta}\qquad\text{and}\qquad \sigma < \frac{c_3}{r}.
\end{align}
\end{theorem}
        
        \begin{remark}[\jie{Sketch proof of Theorem~\ref{thm_main}}] \label{remark2}
            The proof of Theorem~\ref{thm_main} is nontrivial and is included in the Appendix. Next, we briefly explain the sketch of the proof. First, we will define what we refer to as $D_1$-distance and $D_2$-distance between $\left(\bm W, \bm a, \bm b\right)$ and $\left(\bm W^{\star}, \bm a^{\star}, \bm b^{\star}\right)$. These distances can be regarded as the counterpart of the classical $\ell_1$ and $\ell_2$ distances between two vectors, but allowing the invariance under any permutation of neurons (Remark~\ref{remark1}). 
            Then, we let 
            \begin{align}
            \Delta_n \defn \frac{1}{n} \sum_{i = 1}^n\biggl[
            &\sum_{j = 1}^r a_j\mathrm{relu}(\bm w_j^{\T }\bm x_i + b_j) \nonumber\\
            &- \sum_{j = 1}^r a_j^{\star}\mathrm{relu}(\bm w_j^{\star\T }\bm x_i + b_j^{\star})\biggr]^2,  \nonumber
            \end{align}
            where $(\bm W, \bm a, \bm b)$ is the solution of the problem in \ref{eq:l1}, and develop the following upper and lower bounds of it:
            \begin{align}
                &\Delta_n \le c_6\left(\frac{r}{S}+\frac{r\log^3 \frac{p}{n\delta}}{n}\right)D_1^2 + c_6\sigma^2 \quad \textrm{and} \nonumber \\
                &\Delta_n 
                \ge c_4\min\left\{\frac{1}{r}, D_2^2\right\} \label{eq_ineq}
            \end{align}
            hold with probability at least $1 - \delta$, provided that $n \ge c_5 S^3r^4 \log^4\frac{p}{\delta}$, for some constants $c_4,c_5,c_6$, and $S$ to be specified. 
            Here, the upper bound will be derived from a series of elementary inequalities. The lower bound is reminiscent of the Restricted Isometry Property (RIP)~\cite{candes2005decoding} for linear models. We will derive it from the lower bound of the population counterpart by concentration arguments, namely 
            \begin{align}
                &\mathbb{E}\left[\sum_{j = 1}^r a_j\mathrm{relu}(\bm w_j^{\T }\bm x + b_j) - \sum_{j = 1}^r a_j^{\star}\mathrm{relu}(\bm w_j^{\star\T }\bm x + b_j^{\star})\right]^2 \nonumber \\
                &\ge c\min\left\{\frac{1}{r}, D_2^2\right\}, \nonumber
            \end{align}
            for some constant $c > 0$.
            The bounds in \ref{eq_ineq} imply that with high probability,
            \begin{align*}
                c_4\min\left\{\frac{1}{r}, D_2^2\right\} \le c_6\left(\frac{r}{S}+\frac{r\log^3 \frac{p}{n\delta}}{n}\right)D_1^2 + c_6\sigma^2,
            \end{align*}
            Using this and an inequality connecting $D_1$ and $D_2$, we can prove the final result.  

        \end{remark}
            
    \jie{
    \begin{remark}[Alternative assumption and result]
    We provide an alternative to Theorem~\ref{thm_main_elab}. Consider the following Assumption~1' as an alternative to Assumption~1.

    \noindent\textbf{Assumption 1'}.
    For some constant $B > 0$,  
        \begin{align}
            \|\bm w_j^{\star}\|_2 \le B\quad\text{and}\quad|b^{\star}_j| \le B\quad \text{ for all }1\le j \le r. \nonumber
        \end{align}
        In addition, 
        \begin{align}
            &\mathbb{E}\Big[\langle \bm a, \mathrm{relu}(\bm W^{\top}\bm x + \bm b) \rangle - \langle \bm a^{\star}, \mathrm{relu}(\bm W^{\star\top}\bm x + \bm b^{\star}) \rangle\Big]^2 \nonumber \\
            &\ge \psi D_2\left[(\bm W, \bm a, \bm b), (\bm W^{\star}, \bm a^{\star}, \bm b^{\star})\right]^2, \label{eq:psi}
        \end{align}
    and 
    \begin{align} \label{eq:samples-psi}
        n > \frac{c_2}{\psi}s^3r^3\log^4\frac{p}{\delta}.
    \end{align}
    
    With Assumption~1' instead of Assumption~\ref{assumption2}, one can still derive the same result as Theorem~\ref{thm_main_elab}.
    The proof of the above result is similar to that of Theorem~\ref{thm_main_elab}, except that we insert Inequality~\eqref{eq:psi} instead of Inequality~\eqref{eq:population} into~\eqref{eq:empirical} in Appendix~\ref{append_main}.
    \end{remark}
    }
\subsection{Variable selection}\label{sec_alg}

    To solve the optimization problem~\ref{eq:l1} in practice, we consider the following alternative problem,
    \begin{align} 
        \min_{\bm W \in \R^{p \times r}, \bm a \in \R^{r}, \bm b \in \R^{r}} \biggl\{ 
        &\frac{1}{n}\sum_{i = 1}^n \biggl(y_i - \sum_{j = 1}^r a_j\cdot \mathrm{relu}(\bm w_j^{\T }\bm x_i + b_j)\biggr)^2 \nonumber \\ 
        & + \lambda\norm{\W}_1 \biggr\}. \label{eq:penalizedNN}
    \end{align}
    It has been empirically shown that algorithms such as the stochastic gradient descent can find a good approximate solution to the above optimization problem~\cite{bottou2010large, kingma2014adam}.
    Next, we will discuss some details regarding the variable selection using LASSO-regularized neural networks. 
    
    \textbf{Tuning parameters.}
    Given a labeled dataset in practice, we will need to tune hyper-parameters including the penalty term $\lambda$, the number of neurons $r$, learning rate, and the number of epochs. We suggest the usual approach that splits the available data into training and validation parts. The training data are used to estimate neural networks for a set of candidate hyper-parameters. The most suitable candidate will be identified based on the predictive performance on the validation data. 
    
    
    \textbf{Variable importance.} 
        Inspired by Corollary~\ref{coro_selection}, we interpret the $\ell_2$-norm of $\widehat{\bm w}_{i,\cdot}$ as the importance of the $i$-th variable, for $i=1,\ldots,p$. 
        Corollary~\ref{coro_selection} indicates that we can accurately identify all the significant variables in $\S^{\star}$ with high probability if we correctly set the cutoff value $c_1 \sigma$.
        
        \textbf{Setting the cutoff value.}
        It is conceivable that variables with large importance are preferred over those with near-zero importance. This inspires us to cluster the variables into two groups based on their importance. Here, we suggest two possible approaches.
        The first is to use a data-driven approach such as \textit{k}-means and Gaussian mixture model (GMM). 
        The second is to manually set a threshold value according to domain knowledge on the number of important variables. 
            
        \textbf{Extension to deep neural networks.} Inspired by~\ref{eq:penalizedNN}, we can intuitively generalize the proposed method to deep neural networks
        by penalizing the $\ell_1$-norm of the weight matrix in the input layer. Though we do not have a theoretical analysis for this broader setting, numerical studies show that it is still effective.
\section{Experiments}\label{sec_exp}

    We perform experimental studies to show the promising performance of the proposed variable selection method. We compare the variable selection accuracy and prediction performance of the proposed algorithm (`NN') with several baseline methods, including LASSO (`LASSO'), orthogonal matching pursuit (`OMP'), random forest (`RF'), gradient boosting (`GB'), neural networks with group LASSO (`GLASSO')~\cite{wang2017convergence}, group sparse regularization (`GSR')~\cite{scardapane2017group}, and LNET (`LNET')~\cite{lemhadri2021lassonet}. 
    The `NN' hyperparameters to search over are the penalty term $\lambda \in \{0.1, 0.05, 0.01, 0.005, 0.001\}$, the number of neurons $r \in \{20, 50, 100\}$, the learning rate in $\{0.01, 0.005, 0.001\}$, and the number of epochs in $\{100, 200, 500\}$. Moreover, we extend `NN' to a neural network that contains an additional hidden layer of ten neurons. We distinguish the proposed method with two-layer and three-layer neural networks by `NN-2' and `NN-3', respectively.
    Further experimental details are included in Appendix~\ref{appendix_exp}.
    
    \subsection{Synthetic datasets}
    
        \subsubsection{NN-generated dataset}\label{subsubsec_sim1}
        The first experiment uses the data generated from \Autoref{eq_form} with $p=100$ variables, $r=16$ neurons.
        The first $10$ rows of neural coefficients $\bm W$ are independently generated from the standard uniform distribution, and the remaining rows are zeros, representing $10$ significant variables. The neural biases $\bm b$ are also generated from the standard uniform distribution. 
        The signs of neurons, $\bm a$, follow an independent Bernoulli distribution. The training size is $n=500$, and the test size is $2000$.
        The noise is zero-mean Gaussian with standard deviation $\sigma$ set to be $0$, $0.5$, $1$, and $5$. For each $\sigma$, we evaluate its mean squared error on the test dataset and three quantities for variable selection: the number of correctly selected variables (`TP', the larger the better), wrongly selected variables (`FP', the smaller the better), and area-under-curve score (`AUC', the larger the better). Here, `AUC' is evaluated based on the variable importance given by each method, which is detailed in Appendix~\ref{appendix_exp}. The procedure is independently replicated $20$ times. 
        
        The results are reported in \Autoref{tab:num_snr_nn} and \Autoref{tab:mse_snr_nn}, which suggest that `NN' has the best overall performance for both selection and prediction. In particular, `NN-2' and `NN-3' have almost the same performance among all situations, which empirically 
 demonstrates that the proposed method also works for deeper neural networks. 
        It is interesting to compare `NN' with `LNET': `NN' has slightly higher test error than `LNET' when the noise level is small, but a much smaller false positive rate and higher AUC score than `LNET'. It indicates that `NN' is more accurate for variable selection, while `LNET' tends to over-select variables for better prediction accuracy.
        Also, all the methods perform worse as the noise level $\sigma$ increases. 
        
           \begin{table}[tb]
           \begin{minipage}[t]{1\linewidth}
             \centering
                \caption{Performance comparison on the NN-generated data, in terms of the number of correctly (`TP'),  wrongly (`FP') selected features, and the AUC score for different $\sigma$. The standard errors are within the parentheses. }
                \label{tab:num_snr_nn}
                \vspace{.1cm}
                \resizebox{\columnwidth}{!}{%
                \begin{tabular}{lrrrrr}
                 \toprule 
                 \textbf{Method} & \textbf{Metric} & \textbf{$\sigma=0$} & \textbf{$\sigma=0.5$} & \textbf{$\sigma=1$} & \textbf{$\sigma=5$} \\
                 \midrule
LASSO & TP &    8.30 (2.57) &    9.10 (1.22) &    9.20 (1.03) &    6.70 (3.68) \\
         & FP &   10.40 (6.99) &   10.70 (6.09) &   13.35 (6.17) &    8.60 (8.13) \\
         & AUC &    0.96 (0.06) &    0.96 (0.05) &    0.96 (0.05) &    0.87 (0.14) \\
OMP & TP &    8.45 (1.53) &    8.30 (1.35) &    8.00 (2.17) &    6.00 (2.12) \\
         & FP &    0.10 (0.30) &    0.15 (0.36) &    0.25 (0.70) &    0.65 (1.06) \\
         & AUC &    0.92 (0.08) &    0.91 (0.07) &    0.91 (0.09) &    0.80 (0.11) \\
RF & TP &    6.95 (3.12) &    5.55 (3.29) &    5.75 (3.27) &    4.20 (2.60) \\
         & FP &    0.45 (0.67) &    0.40 (0.66) &    0.35 (0.57) &    1.20 (2.18) \\
         & AUC &    0.99 (0.02) &    0.97 (0.03) &    0.95 (0.04) &    0.86 (0.12) \\
GB & TP &    6.85 (2.97) &    7.15 (3.09) &    5.75 (3.69) &    5.65 (3.20) \\
         & FP &    1.35 (1.46) &    1.60 (1.91) &    2.10 (2.07) &    5.20 (6.69) \\
         & AUC &    0.98 (0.02) &    0.97 (0.03) &    0.97 (0.03) &    0.88 (0.11) \\

GLASSO & TP &    9.35 (1.42) &    9.80 (0.51) &    9.45 (0.59) &    6.35 (2.22) \\
         & FP &    0.10 (0.44) &    0.65 (0.96) &    1.00 (1.38) &   8.50 (11.45) \\
         & AUC &    1.00 (0.00) &    1.00 (0.00) &    0.99 (0.02) &    0.84 (0.14) \\
GSR & TP &    9.55 (1.96) &   10.00 (0.00) &    9.90 (0.30) &    7.90 (2.81) \\
         & FP &    1.25 (2.05) &    0.70 (1.19) &    1.35 (2.26) &  22.45 (29.67) \\
         & AUC &    1.00 (0.00) &    1.00 (0.00) &    1.00 (0.01) &    0.84 (0.18) \\
LNET & TP &   10.00 (0.00) &   10.00 (0.00) &    9.95 (0.22) &    5.10 (2.91) \\
         & FP &  66.40 (11.15) &  59.00 (14.64) &  41.60 (21.62) &   8.05 (14.41) \\
         & AUC &    0.63 (0.06) &    0.67 (0.08) &    0.77 (0.12) &    0.83 (0.12) \\
NN-2 & TP &   10.00 (0.00) &    9.85 (0.36) &    9.80 (0.51) &    7.80 (1.96) \\
         & FP &    0.75 (1.13) &    1.25 (2.09) &    2.55 (5.80) &  12.30 (10.57) \\
         & AUC &    \textbf{1.00 (0.00)} &    \textbf{1.00 (0.01)} &    \textbf{1.00 (0.01)} &    \textbf{0.88 (0.13)} \\
NN-3 & TP &   9.75 (0.16) &   9.85 (0.08) &   9.10 (0.45) &  6.95 (0.44) \\
   & FP &   0.75 (0.34) &   0.65 (0.26) &   0.45 (0.18) &  2.50 (0.73) \\
   & AUC &   1.00 (0.00) &   1.00 (0.00) &   0.99 (0.00) &  0.85 (0.03) \\         
                \bottomrule
                \end{tabular}
                }
        \end{minipage}        \hfill
        \begin{minipage}[t]{1\linewidth}            
                \centering
                \caption{Performance comparison on the NN-generated data, in terms of the average mean squared error for different $\sigma$.}
                \label{tab:mse_snr_nn}
                \vspace{.1cm}
                \resizebox{\columnwidth}{!}{%
                \begin{tabular}{lrrrr}
                    \toprule 
                    \textbf{Method} &
                 { $\sigma=0$}   & \textbf{$\sigma=0.5$} & \textbf{$\sigma=1$} & \textbf{$\sigma=5$} \\
                    \midrule
LASSO  &   5.15 (0.64) &  6.49 (1.08) &  4.25 (0.49) &   7.42 (0.79) \\
OMP    &   5.08 (0.72) &  6.35 (1.06) &  5.48 (1.14) &   6.41 (0.63) \\
RF     &  12.62 (2.57) &  7.07 (1.34) &  9.89 (2.19) &  15.94 (3.60) \\
GB     &   8.19 (2.20) &  4.06 (0.70) &  7.08 (1.17) &  10.60 (1.56) \\
GLASSO &   1.09 (0.06) &  1.19 (0.07) &  1.79 (0.08) &   9.14 (0.32) \\
GSR    &   0.57 (0.03) &  0.64 (0.04) &  0.95 (0.04) &   5.43 (0.37) \\
LNET   &   \textbf{0.51 (0.02)} & \textbf{ 0.63 (0.02)} &  0.99 (0.03) &   5.10 (0.34) \\
NN-2   & 0.67 (0.04) &0.74 (0.04) &1.06 (0.05) &\textbf{3.91 (0.22)}\\
NN-3     &   0.62 (0.03) &  0.77 (0.04) &  \textbf{0.87 (0.04)} &   4.05 (0.25) \\
                    \bottomrule
                \end{tabular}
                }
        \end{minipage}        
        \end{table}
        
        \subsubsection{Linear dataset} \label{subsubsec_sim2}
        This experiment considers data generated from a linear model 
        $y = \x^\T \bm \beta+ \xi,$
        where $\bm \beta=(3,1.5,0,0,2,0,0,0)^\T $, $\xi \sim \N(0, \sigma^2)$, and $\bm \x$ follows a multivariate Gaussian distribution whose $(i,j)$-th correlation is $0.5^{\abs{i-j}}$. Among the $p=8$ features, only three of them are significant. 
        The training size is $n=60$, and the test size is $200$. The other settings are the same as Subsubsection~\ref{subsubsec_sim1}.
        The results are presented in \Autoref{tab:num_snr_lin, tab:mse_snr_lin}. 
        
        The results show that the linear model-based methods `LASSO' and `OMP' have the best overall performance, which is expected since the underlying data are from a linear model. The proposed `NN' approach is almost as good as the linear methods. Note that `NN-3' outperforms `NN-2' in this case. One possible explanation is that deeper neural networks have much larger expressivity than two-layer networks. On the other hand, the tree-based methods `RF' and `GB' perform significantly worse. 
        This is possibly because the sample size $n=60$ is relatively small, so the tree-based methods have a large variance. Meanwhile, the `NN' uses the $\ell_1$ penalty to alleviate the over-parameterization and consequently spots the relevant variables. Additionally, `NN' exhibits a positive association between the selection accuracy and prediction performance, while the tree-based methods do not.
        
           \begin{table}[tb]
                    \begin{minipage}[t]{1\linewidth}
                \centering
                \caption{Performance comparison on the linear data, in terms of the number of correctly (`TP'),  wrongly (`FP') selected features, and the AUC score for different $\sigma$.}
                \label{tab:num_snr_lin}
                \vspace{.1cm}
                \resizebox{\columnwidth}{!}{%
                \begin{tabular}{lrrrrr}
                    \toprule 
                     \textbf{Method} & & \textbf{$\sigma=0$} & \textbf{$\sigma=1$} & \textbf{$\sigma=3$} & \textbf{$\sigma=5$} \\
                     \midrule
LASSO & TP &  3.00 (0.00) &  3.00 (0.00) &  2.85 (0.11) &  2.05 (0.18) \\
     & FP &  0.00 (0.00) &  0.00 (0.00) &  0.00 (0.00) &  0.50 (0.19) \\
     & AUC &  1.00 (0.00) &  1.00 (0.00) & 1.00 (0.00) &  0.90 (0.03) \\
OMP & TP &  3.00 (0.00) &  2.90 (0.10) &  2.95 (0.05) &  2.15 (0.18) \\
     & FP &  0.00 (0.00) &  0.00 (0.00) &  0.00 (0.00) &  0.30 (0.12) \\
     & AUC &  1.00 (0.00) &  1.00 (0.00) &  1.00 (0.00) &  0.88 (0.03) \\
RF & TP &  1.15 (0.08) &  1.30 (0.12) &  1.25 (0.12) &  1.60 (0.16) \\
     & FP &  0.00 (0.00) &  0.00 (0.00) &  0.05 (0.05) &  0.20 (0.11) \\
     & AUC &  1.00 (0.00) &  0.99 (0.01) &  0.99 (0.00) &  0.83 (0.03) \\
GB & TP &  1.35 (0.16) &  1.35 (0.16) &  1.30 (0.14) &  1.90 (0.20) \\
     & FP &  0.00 (0.00) &  0.00 (0.00) &  0.00 (0.00) &  0.30 (0.17) \\
     & AUC &  1.00 (0.00) &  0.99 (0.01) &  0.99 (0.01) &  0.91 (0.02) \\
GLASSO & TP &  2.80 (0.13) &  2.70 (0.16) &  2.40 (0.19) &  1.95 (0.18) \\
     & FP &  0.05 (0.05) &  0.05 (0.05) &  0.05 (0.05) &  0.70 (0.20) \\
     & AUC &  1.00 (0.00) &  1.00 (0.00) &  0.99 (0.00) &  0.80 (0.04) \\
GSR & TP &  2.90 (0.10) &  2.90 (0.07) &  2.80 (0.13) &  1.90 (0.16) \\
     & FP &  0.00 (0.00) &  0.10 (0.10) &  0.00 (0.00) &  0.55 (0.15) \\
     & AUC &  1.00 (0.00) &  1.00 (0.00) &  1.00 (0.00) &  0.84 (0.03) \\
LNET & TP &  3.00 (0.00) &  3.00 (0.00) &  2.85 (0.15) &  1.70 (0.19) \\
     & FP &  0.00 (0.00) &  0.20 (0.09) &  0.95 (0.26) &  0.55 (0.23) \\
     & AUC &  1.00 (0.00) &  0.98 (0.01) &  0.88 (0.03) &  0.80 (0.04) \\
NN-2 & TP &  2.50 (0.17) &  2.40 (0.17) &  2.55 (0.14) &  2.25 (0.17) \\
         & FP &  0.05 (0.05) &  0.20 (0.15) &  0.25 (0.12) &  0.75 (0.18) \\
         & AUC &  0.99 (0.00) &  0.99 (0.00) &  0.99 (0.00) &  0.89 (0.2) \\
NN-3 & TP &  3.00 (0.00) &  2.65 (0.15) &  2.90 (0.07) &  2.10 (0.20) \\
     & FP &  0.00 (0.00) &  0.00 (0.00) &  0.00 (0.00) &  0.35 (0.15) \\
     & AUC &  \textbf{1.00 (0.00)} &  \textbf{1.00 (0.00)} &  \textbf{1.00 (0.00)} &  \textbf{0.92 (0.02) }\\
                    \bottomrule
                \end{tabular}
                }
            \end{minipage}\hfill
            \begin{minipage}[t]{1\linewidth}
             \centering
                \caption{Performance comparison on the linear data, in terms of the number of average mean squared error for different $\sigma$.}
                \label{tab:mse_snr_lin}
                \vspace{.1cm}
                \resizebox{\columnwidth}{!}{%
                \begin{tabular}{lrrrr}
                    \toprule 
                    \textbf{Method} & \textbf{$\sigma=0$} & \textbf{$\sigma=1$} & \textbf{$\sigma=3$} & \textbf{$\sigma=5$} \\
                    \midrule
LASSO  &  0.00 (0.00) &  0.04 (0.01) &  0.17 (0.03) &   4.09 (0.45) \\
OMP    &  \textbf{0.00 (0.00)} &  \textbf{0.02 (0.00)} &  \textbf{0.09 (0.01)} &   4.19 (0.44) \\
RF     &  3.57 (0.22) &  3.58 (0.21) &  3.52 (0.16) &   7.89 (0.55) \\
GB     &  2.45 (0.15) &  3.04 (0.17) &  3.01 (0.17) &  11.59 (0.81) \\
GLASSO &  0.10 (0.01) &  0.21 (0.02) &  0.34 (0.03) &   4.95 (0.32) \\
GSR   &  0.09 (0.01) &  0.18 (0.02) &  0.30 (0.04) &   4.16 (0.34) \\
LNET &  0.18 (0.02) &  0.17 (0.02) &  0.34 (0.04) &   3.70 (0.59) \\
NN-2       &  0.09 (0.01) &  0.19 (0.02) &  0.37 (0.04) &   3.96 (1.37) \\
NN-3   &  0.03 (0.00) &  0.10 (0.02) &  0.17 (0.02) &   \textbf{3.17 (0.45)} \\
                    \bottomrule
                \end{tabular}
                }
                \end{minipage}
            \end{table}            
        
        \subsubsection{Friedman dataset}
        This experiment uses the Friedman dataset with the following nonlinear data-generating process,
        $ 
            y = 10 \sin(\pi x_1 x_2) + 20 (x_3 - 0.5)^ 2 + 10 x_4 + 5 x_5
        + \xi
        $. 
        We generate standard Gaussian predictors $\x$ with a dimension of $p=50$. The training size is $n=500$ and the test size is $2000$. Other settings are the same as before. The results are summarized in \Autoref{tab:num_snr, tab:mse_snr}. 
        For this nonlinear dataset, `NN' and `GB' accurately find the significant variables and exclude redundant ones, while the linear methods fail to select the quadratic factor $x_3$. 
        As for the prediction performance, neural network-based methods outperform other methods.
        In particular, `NN' is better than `GLASSO' and `GSR', while `LNET' exhibits better prediction and worse selection performance as seen in previous experiments.
        
           \begin{table}[tb]
                    \begin{minipage}[t]{1\linewidth}
                \centering
                \caption{Performance comparison on the Friedman data, in terms of the number of correctly (`TP'),  wrongly (`FP') selected features, and the AUC score for different $\sigma$.}
                \label{tab:num_snr}
                \vspace{.1cm}
                \resizebox{\columnwidth}{!}{%
                \begin{tabular}{lrrrrr}
                    \toprule 
                     \textbf{Method} & & \textbf{$\sigma=0$} & \textbf{$\sigma=0.5$} & \textbf{$\sigma=1$} & \textbf{$\sigma=5$} \\
                     \midrule
LASSO & TP &   4.05 (0.05) &   4.00 (0.00) &   4.10 (0.07) &  4.00 (0.10) \\
     & FP &   1.45 (0.47) &   1.85 (0.49) &   2.00 (0.42) &  3.10 (0.61) \\
     & AUC &   0.91 (0.01) &   0.91 (0.01) &   0.91 (0.01) &  0.90 (0.01) \\
OMP & TP &   4.00 (0.00) &   4.00 (0.00) &   4.00 (0.00) &  3.80 (0.15) \\
     & FP &   0.10 (0.07) &   0.10 (0.07) &   0.10 (0.07) &  0.05 (0.05) \\
     & AUC &   0.90 (0.00) &   0.90 (0.00) &   0.90 (0.00) &  0.89 (0.01) \\
RF & TP &   4.60 (0.27) &   4.60 (0.27) &   4.80 (0.19) &  4.10 (0.29) \\
     & FP &   0.10 (0.07) &   0.00 (0.00) &   0.05 (0.05) &  0.25 (0.12) \\
     & AUC &   1.00 (0.00) &   1.00 (0.00) &   1.00 (0.00) &  1.00 (0.00) \\
GB & TP &   4.70 (0.21) &   4.90 (0.10) &   4.80 (0.19) &  4.30 (0.25) \\
     & FP &   0.00 (0.00) &   0.00 (0.00) &   0.05 (0.05) &  0.40 (0.18) \\
     & AUC &   1.00 (0.00) &   1.00 (0.00) &   1.00 (0.00) &  \textbf{1.00 (0.00)} \\
GLASSO & TP &   4.80 (0.15) &   4.80 (0.09) &   4.35 (0.20) &  3.80 (0.09) \\
     & FP &   0.05 (0.05) &   0.25 (0.12) &   0.10 (0.10) &  0.95 (0.37) \\
     & AUC &   1.00 (0.00) &   0.99 (0.01) &   0.99 (0.00) &  0.88 (0.01) \\
GSR & TP &   4.20 (0.22) &   4.60 (0.18) &   4.70 (0.16) &  4.00 (0.07) \\
     & FP &   0.15 (0.11) &   0.35 (0.11) &   0.25 (0.10) &  2.30 (0.52) \\
     & AUC &   1.00 (0.00) &   1.00 (0.00) &   1.00 (0.00) &  0.91 (0.01) \\
LNET & TP &   4.75 (0.24) &   4.75 (0.24) &   5.00 (0.00) &  3.45 (0.17) \\
     & FP &  21.20 (2.00) &  25.90 (1.97) &  31.35 (1.55) &  2.15 (1.24) \\
     & AUC &   0.74 (0.02) &   0.69 (0.02) &   0.65 (0.02) &  0.91 (0.02) \\
NN-2 & TP &   4.80 (0.13) &   4.40 (0.20) &   4.80 (0.13) &  4.20 (0.13) \\
     & FP &   0.50 (0.22) &   0.50 (0.26) &   0.60 (0.29) &  1.25 (0.32) \\
     & AUC &   1.00 (0.00) &   1.00 (0.00) &   1.00 (0.00) &  0.95 (0.01) \\
NN-3 & TP &  4.85 (0.08) &  4.90 (0.07) &  4.85 (0.08) &  4.45 (0.11) \\
     & FP &  0.35 (0.16) &  0.25 (0.12) &  0.55 (0.22) &  3.85 (0.70) \\
     & AUC &  \textbf{1.00 (0.00)} &  \textbf{1.00 (0.00)} &  \textbf{1.00 (0.00)} &  0.96 (0.01) \\
                    \bottomrule
                \end{tabular}
                }
            \end{minipage}\hfill
            \begin{minipage}[t]{1\linewidth}
                \centering
                \caption{Performance comparison on the Friedman data, in terms of the average mean squared error for different $\sigma$.}
                \label{tab:mse_snr}
                \vspace{.1cm}
                \resizebox{\columnwidth}{!}{%
                \begin{tabular}{lrrrr}
                    \toprule 
                    \textbf{Method} & $\sigma=0$ & $\sigma=0.5$ & $\sigma=1$ & $\sigma=5$ \\
                    \midrule
LASSO  &  6.16 (0.05) &  6.17 (0.06) &  6.07 (0.05) &  6.97 (0.12) \\
OMP    &  5.95 (0.07) &  6.01 (0.07) &  6.04 (0.06) &  6.41 (0.16) \\
RF     &  5.21 (0.05) &  5.19 (0.09) &  5.30 (0.07) &  7.81 (0.15) \\
GB     &  2.70 (0.05) &  2.74 (0.04) &  2.82 (0.06) &  6.33 (0.15) \\
GLASSO &  4.49 (0.12) &  4.64 (0.17) &  5.43 (0.14) &  9.89 (0.23) \\
GSR    &  1.57 (0.05) &  1.73 (0.07) &  2.31 (0.11) &  7.49 (0.19) \\
LNET   & \textbf{ 1.05 (0.19)} &  \textbf{1.29 (0.26) }&  \textbf{1.52 (0.12)} &  9.08 (0.47) \\
NN-2   &  1.58 (0.04) &  1.71 (0.06) &  2.14 (0.06) &  \textbf{5.86 (0.11)} \\
NN-3 &  1.44 (0.04) &  1.62 (0.04) &  1.92 (0.04) &  6.02 (0.16) \\
                    \bottomrule
                \end{tabular}
                }
                \end{minipage}
            \end{table}

    \subsection{BGSBoy dataset}
        The BGSBoy dataset involves 66 boys from the Berkeley guidance study (BGS) of children born in 1928-29 in Berkeley, CA~\cite{tuddenham1954physical}. The dataset includes the height (`HT'), weight (`WT'), leg circumference (`LG'), strength (`ST') at different ages (2, 9, 18 years), and body mass index (`BMI18').
        We choose `BMI18' as the response, which is defined as follows.
        \begin{equation} \label{eq:bmi}
            \textrm{BMI18} = \textrm{WT18}/(\textrm{HT18}/100)^2,
        \end{equation}
        where WT18 and HT18 denote the weight and height at the age of 18, respectively. 
        In other words, `WT18' and `HT18' are sufficient for modeling the response among $p=10$ variables. Other variables are correlated but redundant.
        The training size is $n=44$ and the test size is $22$. 
        Other settings are the same as before. We compare the prediction performance and explore the three features which are most frequently selected by each method. The results are summarized in \Autoref{tab:mse_bgs}. 

    From the results, both linear and NN-based methods can identify `WT18' and `HT18' in all the replications. Meanwhile, tree-based methods may miss `HT18' but select `LG18' instead, which is only correlated with the response.
    Interestingly, we find that the linear methods still predict well in this experiment. A possible reason is that \Autoref{eq:bmi} can be well-approximated by a first-order Taylor expansion on `HT18' at the value of around $180$, and the range of `HT18' is within a small interval around $180$. 
        \begin{table*}[tb]
                \centering
                \caption{Experiment results of different methods on the BGSBoy dataset. RMSE: the mean of the root mean squared error(standard error). Top 3 features: the feature name(number of selection, out of 20 times).}
                \label{tab:mse_bgs}
                \vspace{.1cm}
                \resizebox{2\columnwidth}{!}{%
                \begin{tabular}{lrrrrrrrrr}
                    \toprule 
                    \textbf{Method} & \textbf{LASSO} & \textbf{OMP} & \textbf{RF} & \textbf{GB}    & \textbf{GLASSO}  & \textbf{GSR}  & \textbf{LNET} & \textbf{NN-2} & \textbf{NN-3}\\
                    \midrule
                    \textbf{RMSE} &   0.05 (0.00) &  0.04 (0.00) &  3.10 (0.37) &  2.32 (0.30) &  0.15 (0.06) &  0.09 (0.03) &  \textbf{0.03 (0.00)} &  0.06 (0.02) &  0.12 (0.02)\\
                     \\
                     \multirow{3}{*}{\shortstack[l]{\textbf{Top 3 }\textbf{ frequently}\\\textbf{selected}\textbf{ features}}} 
  &  WT18 (20) &  WT18 (20) &  WT18 (20) &  WT18 (20) &  WT18 (20) &  WT18 (20) &  WT18 (20) &  WT18 (20) &  WT18 (20) \\
{} &  HT18 (20) &  HT18 (20) &  LG18 (20) &  LG18 (20) &  HT18 (20) &  HT18 (20) &  HT18 (20) &  HT18 (20) &  HT18 (20) \\
{} &  HT9  (16) &   ST18 (7) &  HT18 (10) &  HT18 (19) &   LG18 (8) &  LG18 (12) &   HT9  (9) &  HT9  (12) &  LG18 (17) \\
                    \bottomrule
                \end{tabular}}
            \end{table*}
    
    \subsection{UJIIndoorLoc dataset}
    
        The UJIINdoorLoc dataset aims to solve the indoor localization problem via WiFi fingerprinting and other variables such as the building and floor numbers. A detailed description can be found in \cite{torres2014ujiindoorloc}. Specifically, we have 520 Wireless Access Points (WAPs) signals (which are continuous variables) and `FLOOR', `BUILDING', `SPACEID', `RELATIVEPOSITION', `USERID', and `PHONEID' as categorical variables. The response variable is a user's longitude (`Longitude'). The dataset has $19937$ observations. We randomly sample $3000$ observations and split them into $n=2000$ for training and $1000$ for test. As part of the pre-processing, we create binary dummy variables for the categorical variables, which results in $p=681$ variables in total. We explore the ten features that are most frequently selected by each method. We set the cutoff value as the tenth-largest variable importance. The procedure is independently replicated $100$ times. The results are reported in \Autoref{tab:mse_uji}.
        
        Based on the results, the `NN' achieves the best prediction performance and significantly outperforms other methods. 
        As for variable selection, since `BUILDING' greatly influences the location from our domain knowledge, it is non-surprisingly selected by all methods in every replication. However, except for `BUILDING', different methods select different variables with some overlapping, e.g., `PHONEID\_14' selected by `GLASSO' and `GB', `USERID\_16' selected by `NN' and `LASSO', which indicate the potentially important variables. `LNET' again selects more variables than other methods. There are nearly 60 variables selected by `LNET' in every replication.
        Nevertheless, those methods do not achieve an agreement for variable selection. `NN' implies that all the WAPs signals are weak while categorical variables provide more information about the user location. Given the very high missing rate of WAPs signals ($97\%$ on average, as reported in \cite{torres2014ujiindoorloc}), the interpretation of `NN' seems reasonable.
            \begin{table*}[tb]
                \centering
                \caption{Experiment results of different methods on the UJIINdoor dataset. RMSE: the mean of the root mean squared error(standard error). Top 10 features: the feature name(number of selection, out of 20 times).}
                \label{tab:mse_uji}
                \vskip 0.1in
                \resizebox{0.9\textwidth}{!}{%
                \begin{tabular}{lrrrrr}
                    \toprule 
                    \textbf{Method}  & \textbf{LASSO} & \textbf{OMP} & \textbf{RF} & \textbf{GB} & \\
                    \midrule
                    \textbf{RMSE} & 14.20 (0.08) &  16.75 (0.12) &  9.60 (0.12) &  11.02 (0.09) \\
                   \\ \multirow{10}{*}{\shortstack[l]{\textbf{Top 10 }\\\textbf{frequently}\\\textbf{selected}\\\textbf{features}}} 
 &  BUILDINGID\_2 (20) &  BUILDINGID\_2 (20) &  BUILDINGID\_1 (20) &  BUILDINGID\_2 (20) \\
{} &  BUILDINGID\_1 (20) &  BUILDINGID\_1 (20) &  BUILDINGID\_2 (20) &        WAP120 (20) \\
{} &      USERID\_9 (20) &        WAP099 (17) &        WAP141 (16) &  BUILDINGID\_1 (20) \\
{} &     USERID\_16 (20) &     USERID\_10 (17) &        WAP120 (16) &        WAP141 (17) \\
{} &     USERID\_10 (18) &     USERID\_16 (14) &        WAP117 (14) &        WAP099 (16) \\
{} &        WAP099 (18) &      USERID\_7 (13) &        WAP035 (13) &    PHONEID\_14 (14) \\
{} &      USERID\_7 (14) &      USERID\_9 (10) &        WAP173 (13) &        WAP113 (13) \\
{} &        WAP121 (10) &         WAP120 (8) &        WAP167 (10) &        WAP114 (12) \\
{} &         WAP118 (8) &         WAP124 (8) &         WAP118 (8) &        WAP117 (12) \\
{} &         WAP039 (7) &         WAP101 (8) &         WAP113 (8) &         WAP140 (9) \\
\midrule
\textbf{Method}  & \textbf{GLASSO} & \textbf{GSR} & \textbf{LNET} & \textbf{NN-2} & \textbf{NN-3} \\
                    \midrule
                    \textbf{RMSE} & 11.25 (0.14) &  11.44 (0.18) &  11.19 (0.09) &  \textbf{8.86 (0.09)} &  \textbf{8.86 (0.09)} \\
                   \\ \multirow{10}{*}{\shortstack[l]{\textbf{Top 10 }\\\textbf{frequently}\\\textbf{selected}\\\textbf{features}}} 
 &  BUILDINGID\_1 (20) &  BUILDINGID\_1 (20) &  BUILDINGID\_1 (20)&   SPACEID\_202 (20) &   SPACEID\_202 (20) \\
{} &  BUILDINGID\_2 (20) &  BUILDINGID\_2 (20) &  BUILDINGID\_2 (20)&  BUILDINGID\_1 (20) &  BUILDINGID\_1 (20) \\
{} &    PHONEID\_14 (18) &     USERID\_16 (17) &    PHONEID\_22 (20)&  BUILDINGID\_2 (20) &  BUILDINGID\_2 (20) \\
{} &   SPACEID\_202 (17) &   SPACEID\_202 (15) &     PHONEID\_6 (20)&     USERID\_16 (19) &     USERID\_16 (17) \\
{} &      USERID\_8 (16) &     PHONEID\_6 (14) &     PHONEID\_8 (20)&   SPACEID\_203 (16)  &   SPACEID\_203 (13) \\
{} &     USERID\_16 (16) &   SPACEID\_203 (14) &   SPACEID\_103 (20)&   SPACEID\_201 (14)  &   SPACEID\_201 (13)\\
{} &       FLOOR\_3 (12) &      USERID\_7 (13) &   SPACEID\_136 (20)&         WAP140 (7) &       USERID\_9 (8)\\
{} &       USERID\_9 (9) &   SPACEID\_201 (12) &   SPACEID\_201 (20)&         WAP121 (5) &         WAP121 (8)\\
{} &         WAP478 (8) &         WAP141 (6) &   SPACEID\_202 (20) &         WAP030 (5) &         WAP030 (7)\\
{} &         WAP099 (7) &         WAP176 (6) &   SPACEID\_203 (20)&    SPACEID\_224 (5) &    SPACEID\_224 (5)\\
                    \bottomrule
                \end{tabular}}
            \end{table*}

    \subsection{Summary}
            The experiment results show the following points. First, `NN' can stably identify the important variables and have competitive prediction performance compared with the baselines. Second, the increase of the noise level will hinder both the selection and prediction performance. Third, the LASSO regularization is crucial for `NN' to avoid over-fitting, especially for small data. Using group LASSO or a mixed type of penalty has a similar performance as `NN', while `LNET' tends to over-select importance variables. Fourth, the selection and prediction performances are often positively associated for `NN', but may not be the case for baseline methods.
                
        
\section{Concluding Remarks}\label{sec_con}
    We established a theory for the use of LASSO in two-layer ReLU neural networks. In particular, we showed that the LASSO estimator could stably reconstruct the neural network coefficients and identify the critical underlying variables under reasonable conditions. We also proposed a practical method to solve the optimization and perform variable selection. 
    %
    %
    We briefly remark on some interesting further work. \jie{First, a limitation of the work is that we considered only a small $r$. An interesting future problem is to study $r$ that may grow fast with $p$ and $n$.}
    Second, our experiments show that the algorithm can be extended to deeper neural networks. It will be exciting to generalize the main theorem to the multi-layer~cases. 
    
    The Appendix includes proofs and experimental details. 

\section*{Acknowledgment}
The authors would like to thank the anonymous reviewers and Associate Editor for their constructive review comments. The work was supported in part by the U.S. National Science Foundation under Grant Numbers DMS-2134148 and CNS-2220286.




\appendices


\section{Analysis: proof of Theorem~\ref{thm_main_elab}}\label{append_main}
Let $\mathcal{S}$ be the index set with cardinality $S$ consisting of the support for $\bm W^{\star}$ and top entries of $\widehat{\bm W}$, where $S$ will be specified momentarily.
Define
\[
\bm W \defn \widehat{\bm W}_{\mathcal{S}} \in \mathbb{R}^{p \times r}, 
\]
and $a_j = \widehat{a}_j$, $b_j = \widehat{b}_j$.
Define 
\begin{align}
&d_1(\bm w_1, a_1, b_1, \bm w_2, a_2, b_2) \nonumber \\
&= \left\{
\begin{array}{ll}
\|\bm w_1 - \bm w_2\|_1 + |b_1 - b_2| & \text{if }a_1 = a_2; \\
\|\bm w_1\|_1 + \|\bm w_2\|_1 + |b_1| + |b_2| & \text{if }a_1 \ne a_2,
\end{array}
\right.
\end{align}
and
\begin{align}
&d_2(\bm w_1, a_1, b_1, \bm w_2, a_2, b_2) \nonumber \\
&= \left\{
\begin{array}{ll}
\sqrt{\|\bm w_1 - \bm w_2\|_2^2 + |b_1 - b_2|^2} & \text{if }a_1 = a_2; \\
1 & \text{if }a_1 \ne a_2.
\end{array}
\right.
\end{align}
In addition, for permutation $\pi$ on $[r]$, let
\begin{subequations} \label{eq:distance}
\begin{align}
D_1 &\defn \min_{\pi} \sum_{j = 1}^r d_1(\bm w_{\pi(j)}, a_{\pi(j)}, b_{\pi(j)}, \bm w_j^{\star}, a_j^{\star}, b_j^{\star}), \\
D_2 &\defn \min_{\pi} \sqrt{\sum_{j = 1}^r d_2(\bm w_{\pi(j)}, a_{\pi(j)}, b_{\pi(j)}, \bm w_j^{\star}, a_j^{\star}, b_j^{\star})^2}
\end{align}
\end{subequations}
denote the $D_1$-distance and $D_2$-distance between $\left(\bm W, \bm a, \bm b\right)$ and $\left(\bm W^{\star}, \bm a^{\star}, \bm b^{\star}\right)$, respectively.
Then, one has the following bounds.
\begin{lemma} \label{lem:lower}
For any $\bm W \in \mathbb{R}^{p \times r}$ with $\left\|\bm W\right\|_0 \le S$, there exists some universal constants $c_4, c_5 > 0$ such that 
\begin{align}
&\frac{1}{n} \sum_{i = 1}^n\left[\sum_{j = 1}^r a_j\mathrm{relu}(\bm w_j^{\top}\bm x_i + b_j) - \sum_{j = 1}^r a_j^{\star}\mathrm{relu}(\bm w_j^{\star\top}\bm x_i + b_j^{\star})\right]^2 \nonumber \\
& \ge c_4\min\left\{\frac{1}{r}, D_2^2\right\} \label{eq:lower2}
\end{align}
holds with probability at least $1 - \delta$ provided that 
\begin{align}
n \ge c_5S^3r^4\log^4\frac{p}{\delta}.
\end{align}
\end{lemma}

\begin{lemma} \label{lem:upper}
Then, there exists a universal constant $c_6 > 0$ such that
\begin{align}
&\frac{1}{n} \sum_{i = 1}^n\left[\sum_{j = 1}^r a_j\mathrm{relu}(\bm w_j^{\top}\bm x_i + b_j) - \sum_{j = 1}^r a_j^{\star}\mathrm{relu}(\bm w_j^{\star\top}\bm x_i + b_j^{\star})\right]^2 \nonumber \\
&\le c_6\left(\frac{r}{S}+\frac{r\log^3 \frac{p}{n\delta}}{n}\right)D_1^2 + c_6\sigma^2 \nonumber
\end{align}
holds with probability at least $1 - \delta$.
\end{lemma}

By comparing the bounds given in Lemma~\ref{lem:lower} and~\ref{lem:upper}, one has
\begin{align*}
c_4\min\left\{\frac{1}{r}, D_2^2\right\} \le c_6\left(\frac{r}{S}+\frac{r\log^3 \frac{p}{n\delta}}{n}\right)D_1^2 + c_6\sigma^2,
\end{align*}
provided that
\[
n > c_5S^3r^4\log^4\frac{p}{\delta}.
\]
Let $\widehat{\mathcal{S}}^{\star}$ be the index set with cardinality $2s$ consisting of the support for $\bm W^{\star}$ and top entries of $\widehat{\bm W}$.
In addition, let $D_1^{\star}$ and $D_2^{\star}$ denote the $D_1$-distance and $D_2$-distance between $\left(\widehat{\bm W}_{\widehat{\mathcal{S}}^{\star}}, \widehat{\bm a}, \widehat{\bm b}\right)$ and $\left(\bm W^{\star}, \bm a^{\star}, \bm b^{\star}\right)$ in a similar way as~\eqref{eq:distance}.
\jie{Observing the fact that for $S \ge 2s$, one has $\mathcal{S}^{\star} \subset \widehat{\mathcal{S}}^{\star} \subset \mathcal{S}$, we have 
\begin{align*}
    \|\bm{w}_j - \bm{w}_j^{\star}\|_2 \ge \|\bm{w}_{j, \widehat{\mathcal{S}}^{\star}} - \bm{w}_j^{\star}\|_2 = \|\widehat{\bm{w}}_j - \bm{w}_j^{\star}\|_2,
\end{align*}
after some permutation, and then
\begin{align}
    D_2^{\star} \le D_2.\nonumber
\end{align}
In addition, after some permutation, we have $D_1^{\star} \ge \big\|\widehat{\bm W}_{\widehat{\mathcal{S}}^{\star}} - \bm{W}^{\star}\big\|_1 \ge \left\|\bm{W}^{\star}\right\|_1 - \big\|\widehat{\bm{W}}_{\mathcal{S}^{\star}}\big\|_1$ and $\left\|\bm{W}\right\|_1 \le \big\|\widehat{\bm{W}}\big\|_1 \le \left\|\bm{W}^{\star}\right\|_1$.
Then,
\begin{align}
D_1 \le D_1^{\star} + \left\|\widehat{\bm W}_{\mathcal{S}} - \widehat{\bm W}_{\widehat{\mathcal{S}}^{\star}}\right\|_1 \le D_1^{\star} + \big\|\widehat{\bm{W}}\big\|_1 - \big\|\widehat{\bm{W}}_{\mathcal{S}^{\star}}\big\|_1 \le 2D_1^{\star}.\nonumber
\end{align}
}
Combined with Lemma~\ref{lem:D1-D2} in Appendix~\ref{sec_tech}, the above results give
\begin{align*}
D_2^{\star} \le \frac{2c_6}{c_4}\sigma,
\end{align*}
provided that for some constant $c_7 > 0$
\[
n \ge c_5S^3\log^4\frac{p}{\delta} \qquad\text{with }S \ge c_7sr,
\]
such that
\[
c_6\left(\frac{r}{S}+\frac{r\log^3 \frac{p}{n\delta}}{n}\right)D_1^{\star 2} \le \frac{c_4}{8} D_2^{\star 2}.
\]
%
Then, we can conclude the proof since after appropriate permutation,
\[
\|\widehat{\bm W} - \bm W^{\star}\|_{\mathrm{F}} \le 2\|\widehat{\bm W}_{\widehat{\mathcal{S}}^{\star}} - \bm W^{\star}\|_{\mathrm{F}}.
\]

\section{Proof of Lemma~\ref{lem:lower} (lower bound)}

This can be seen from the following three properties.
\begin{itemize}
\item Consider the case that $$D_1 \le \epsilon = \frac{\delta}{4nr}\sqrt{\frac{\pi}{\log\frac{4pn}{\delta}}}.$$ 
With probability at least $1 - \delta$,
\begin{align}
&\frac{1}{n} \sum_{i = 1}^n\biggl[\sum_{j = 1}^r a_j\mathrm{relu}(\bm w_j^{\top}\bm x_i + b_j) \nonumber\\
&\hspace{1cm}- \sum_{j = 1}^r a_j^{\star}\mathrm{relu}(\bm w_j^{\star\top}\bm x_i + b_j^{\star})\biggr]^2 \nonumber\\
&= \frac{D_1^2}{\epsilon^2}\frac{1}{n} \sum_{i = 1}^n\biggl[\sum_{j = 1}^r a_j\mathrm{relu}(\widetilde{\bm w}_j^{\top}\bm x_i + \widetilde{b}_j) \nonumber\\
&\hspace{2cm} - \sum_{j = 1}^r a_j^{\star}\mathrm{relu}(\bm w_j^{\star\top}\bm x_i + b_j^{\star})\biggr]^2, \label{eq:local}
\end{align}
where $\widetilde{\bm w}_j = \bm w_j^{\star} + \frac{\epsilon}{D_1}\left(\bm w_j - \bm w_j^{\star}\right)$ and $\widetilde{b}_j = b_j^{\star} + \frac{\epsilon}{D_1}\left(b_j - b_j^{\star}\right)$.
\item For any $\epsilon > 0$ and 
\[
D_1 \ge \frac{\epsilon}{\sqrt{\frac{S}{n}\log\frac{pr}{S}\log\frac{BS}{\epsilon\delta}}}, 
\]
there exists some universal constant $C_1 > 0$, such that with probability at least $1-\delta$,
\begin{align}
&\frac{1}{n} \sum_{i = 1}^n\biggl[
\sum_{j = 1}^r a_j\mathrm{relu}(\bm w_j^{\top}\bm x_i + b_j)  \nonumber \\
&\hspace{1cm} - \sum_{j = 1}^r a_j^{\star}\mathrm{relu}(\bm w_j^{\star\top}\bm x_i + b_j^{\star})\biggr]^2 \nonumber\\
&\ge \mathbb{E}\left[\sum_{j = 1}^r a_j\mathrm{relu}(\bm w_j^{\top}\bm x + b_j) - \sum_{j = 1}^r a_j^{\star}\mathrm{relu}(\bm w_j^{\star\top}\bm x + b_j^{\star})\right]^2 \nonumber\\
&\qquad- C_1D_1^2\log\frac{pn}{\delta}\sqrt{\frac{S}{n}\log\frac{pr}{S}\log\frac{BS}{\epsilon\delta}}.\label{eq:empirical}
\end{align}
\item For some universal constant $C_2 > 0$
\begin{align}
&\mathbb{E}\left[\sum_{j = 1}^r a_j\mathrm{relu}(\bm w_j^{\top}\bm x + b_j) - \sum_{j = 1}^r a_j^{\star}\mathrm{relu}(\bm w_j^{\star\top}\bm x + b_j^{\star})\right]^2 \nonumber \\
&\ge C_2\min\left\{\frac{1}{r}, D_2^2\right\}. \label{eq:population} 
\end{align}
\end{itemize}

\textit{Putting the above together}.
Let
\[
\epsilon = C_3\frac{\delta}{nr}\sqrt{\frac{S}{n}\log\frac{BnS}{\delta}},
\]
for some universal constant $C_3 > 0$ such that
\[
\frac{\epsilon}{\sqrt{\frac{S}{n}\log\frac{pr}{S}\log\frac{BS}{\epsilon\delta}}} < \frac{\delta}{4nr}\sqrt{\frac{\pi}{\log\frac{4pn}{\delta}}}.
\]
Inserting~\eqref{eq:population} into~\eqref{eq:empirical} gives that
\begin{align}
&\frac{1}{n} \sum_{i = 1}^n\left[\sum_{j = 1}^r a_j\mathrm{relu}(\bm w_j^{\top}\bm x_i + b_j) - \sum_{j = 1}^r a_j^{\star}\mathrm{relu}(\bm w_j^{\star\top}\bm x_i + b_j^{\star})\right]^2 \nonumber\\
&\ge C_2\min\left\{\frac{1}{r}, D_2^2\right\} - C_1D_1^2\log\frac{pn}{\delta}\sqrt{\frac{S}{n}\log\frac{pr}{S}\log\frac{BS}{\epsilon\delta}} \nonumber\\
&\ge \frac{C_2}{2}\min\left\{\frac{1}{r}, D_2^2\right\}, \label{eq:upper-large2}
\end{align}
holds with probability at least $1 - \delta$ provided that for some constant $C_4 > 0$
\begin{align*}
&n \ge C_4S^3r^4\log\frac{pr}{S}\log\frac{BS}{\epsilon\delta}\log^2\frac{pn}{\delta} \quad\text{and} \\
& D_1 \ge \frac{\delta}{4nr}\sqrt{\frac{\pi}{\log\frac{4pn}{\delta}}}. 
\end{align*}
Here, the last line holds due to Lemma~\ref{lem:D1-D2} and we assume that $\max\left\{\|\bm W\|_{\infty}, \|\bm b\|_{\infty}\right\}$ is bounded by some constant.
On the other hand, if $$D_1 < \frac{\delta}{4nr}\sqrt{\frac{\pi}{\log\frac{4pn}{\delta}}},$$ it follows from \eqref{eq:local} and \eqref{eq:upper-large2} that
\begin{align}
&\frac{1}{n} \sum_{i = 1}^n\left[\sum_{j = 1}^r a_j\mathrm{relu}(\bm w_j^{\top}\bm x_i + b_j) - \sum_{j = 1}^r a_j^{\star}\mathrm{relu}(\bm w_j^{\star\top}\bm x_i + b_j^{\star})\right]^2 \nonumber \\
&\ge \frac{D_1^2}{\epsilon^2}\frac{C_2}{2}\min\left\{\frac{1}{r}, \widetilde{D}_2^2\right\} = \frac{C_2}{2}D_2^2.\nonumber
\end{align}
Summing up, we conclude the proof by verifying \eqref{eq:local}, \eqref{eq:empirical}, and \eqref{eq:population} below.

\subsection{Proof of \eqref{eq:local}}

Since $D_1 \le \epsilon = \frac{\delta}{4nr}\sqrt{\frac{\pi}{\log\frac{4pn}{\delta}}}$, without loss of generality, we assume that $a_j = a_j^{\star}$ for $1 \le j \le r$, and 
\[
D_1 = \sum_{j = 1}^r \left(\|\bm w_j - \bm w_j^{\star}\|_1 + |b_j - b_j^{\star}|\right) \le \epsilon.
\]
By taking union bound, with probability at least $1 - \frac{\delta}{2}$, one has for all $1 \le i \le n$ and $1 \le j \le r$,
\[
\left|\bm w_j^{\star\top}\bm x_i + b_j^{\star}\right| > \frac{\delta}{2nr}\sqrt{\frac{\pi}{2}},
\]
since $\|\bm w_j^{\star}\|_2 \ge 1$ and $\bm x_i \sim \mathcal{N}(\zero, \Ind)$.
In addition, for all $1 \le i \le n$ and $1 \le j \le r$, 
\begin{align*}
\left|\bm w_j^{\top}\bm x_i + b_j - \bm w_j^{\star\top}\bm x_i - b_j^{\star}\right| 
&\le \|\bm w_j - \bm w_j^{\star}\|_1\|\bm x_i\|_{\infty} + |b_j - b_j^{\star}| \\
&\le \epsilon\sqrt{2\log\frac{4pn}{\delta}}
\end{align*}
holds with probability at least $1 - \frac{\delta}{2}$.
Here, the last inequality comes from the fact that with probability at least $1-\frac{\delta}{2}$, 
\begin{align}
\|\bm x_i\|_{\infty} \le \sqrt{2\log\frac{4pn}{\delta}}\qquad\text{for all } 1 \le i \le n. \label{eq:x_inf}
\end{align}
Putting together, we have with probability at least $1 - \delta$,
\begin{align} \label{eq:sign}
u(\bm w_j^{\top}\bm x_i + b_j) = u(\bm w_j^{\star\top}\bm x_i + b_j^{\star}),
\end{align}
provided that $$\epsilon \le \frac{\delta}{4nr}\sqrt{\frac{\pi}{\log\frac{4pn}{\delta}}}.$$
Note that $u(x) = 1$ if $x > 0$, and $u(x) = 0$ if $x \le 0$.
Then combining with the definition of $\widetilde{\bm w}_j$ and $\widetilde{b}_j$, the above property yields
\begin{align*}
&\frac{1}{n} \sum_{i = 1}^n\left[\sum_{j = 1}^r a_j\mathrm{relu}(\bm w_j^{\top}\bm x_i + b_j) - \sum_{j = 1}^r a_j^{\star}\mathrm{relu}(\bm w_j^{\star\top}\bm x_i + b_j^{\star})\right]^2 \\
&= \frac{1}{n} \sum_{i = 1}^n\left[\sum_{j = 1}^r a_j^{\star}u(\bm w_j^{\star\top}\bm x_i + b_j^{\star})(\bm w_j^{\top}\bm x_i + b_j - \bm w_j^{\star\top}\bm x_i - b_j^{\star})\right]^2 \\
&= \frac{D_1^2}{\epsilon^2}\frac{1}{n} \sum_{i = 1}^n\biggl[\sum_{j = 1}^r a_j^{\star}u(\bm w_j^{\star\top}\bm x_i + b_j^{\star}) \\
&\hspace{2cm} \times (\widetilde{\bm w}_j^{\top}\bm x_i + \widetilde{b}_j - \bm w_j^{\star\top}\bm x_i - b_j^{\star})\biggr]^2 \\
&= \frac{D_1^2}{\epsilon^2}\frac{1}{n} \sum_{i = 1}^n\biggl[\sum_{j = 1}^r a_j\mathrm{relu}(\widetilde{\bm w}_j^{\top}\bm x_i + \widetilde{b}_j) \\
&\hspace{2cm}  - \sum_{j = 1}^r a_j^{\star}\mathrm{relu}(\bm w_j^{\star\top}\bm x_i + b_j^{\star})\biggr]^2, 
\end{align*}
and the claim is proved.
\jie{Here, the last equality holds due to~\eqref{eq:sign} and $a_j = a_j^{\star}$ for $j = 1, \ldots, r$.}

\subsection{Proof of \eqref{eq:empirical}}

Notice that
\begin{align*}
&\left|a_j\mathrm{relu}(\bm w_j^{\top}\bm x + b_j) - a_j^{\star}\mathrm{relu}(\bm w_j^{\star\top}\bm x_i + b_j^{\star})\right|  \\
&\le \left\{
\begin{array}{ll}
\|\bm w_j - \bm w_j^{\star}\|_1\|\bm x\|_{\infty} + |b_j - b_j^{\star}| & \text{if } a_j = a_j^{\star}, \\
\left(\|\bm w_j\|_1 + \|\bm w_j^{\star}\|_1\right)\|\bm x\|_{\infty} + |b_j| + |b_j^{\star}| & \text{if } a_j \ne a_j^{\star},
\end{array}
\right.
\end{align*}
which leads to
\begin{align}
&\left|\sum_{j = 1}^r a_j\mathrm{relu}(\bm w_j^{\top}\bm x + b_j) - \sum_{j = 1}^r a_j^{\star}\mathrm{relu}(\bm w_j^{\star\top}\bm x + b_j^{\star})\right| \nonumber \\
& \le D_1\max\left\{\|\bm x\|_{\infty}, 1\right\}. \label{eq:crude-bound}
\end{align}
For any fixed $\left(\bm W, \bm a, \bm b\right)$, let 
\[
z_i \defn \sum_{j = 1}^r a_j\mathrm{relu}(\bm w_j^{\top}\bm x_i + b_j) - \sum_{j = 1}^r a_j^{\star}\mathrm{relu}(\bm w_j^{\star\top}\bm x_i + b_j^{\star}),
\]
and define the following event set
\[
\mathcal{E} \defn \left\{\|\bm x_i\|_{\infty} \le \sqrt{2\log\frac{4pn}{\delta}}\qquad\text{for all } 1 \le i \le n\right\}.
\]
Then, with probability at least $1-\delta$,
\begin{align}
&\frac{1}{n} \sum_{i = 1}^n \left(z_i^2 - \mathbb{E}\left[z_i^2\right]\right) \\
&= \frac{1}{n} \sum_{i = 1}^n \biggl\{z_i^2\ind(\mathcal{E}) - \mathbb{E}\left[z_i^2\ind(\mathcal{E})\right] 
 - \mathbb{E}\bigl[z_i^2\ind(\overline{\mathcal{E}})\bigr]\biggr\} \nonumber\\
&\ge -4D_1^2\log\frac{4pn}{\delta}\sqrt{\frac{1}{n}\log\frac{2}{\delta}} - D_1^2\frac{\delta}{n} \nonumber \\
&\ge -5D_1^2\log\frac{4pn}{\delta}\sqrt{\frac{1}{n}\log\frac{2}{\delta}}. \label{eq:random-error-fixed}
\end{align}
Here, the first line holds due to~\eqref{eq:x_inf};
the last line comes from Hoeffding's inequality, and the fact that
\begin{align*}
\left|\mathbb{E}\left[z_i^2\ind(\overline{\mathcal{E}})\right]\right| &\le D_1^2\left|\mathbb{E}\left[\|\bm x_i\|_{\infty}^2\ind(\|\bm x_i\|_{\infty} > \sqrt{2\log\frac{4pn}{\delta}})\right]\right| \\
&\le D_1^2\int_{\sqrt{2\log\frac{4pn}{\delta}}}^{\infty} x^2d \mathbb{P}(\|\bm x_i\|_{\infty} < x) \\
&\le D_1^2\int_{\sqrt{2\log\frac{4pn}{\delta}}}^{\infty} 4xp\exp(-\frac{x^2}{2})dx \le D_1^2\frac{\delta}{n}.
\end{align*}

In addition, consider the following $\epsilon$-net 
\begin{align}
\mathcal{N}_{\epsilon} \defn \biggl\{ &\left(\bm W, \bm a, \bm b\right) : \left|W_{ij}\right| \in \frac{\epsilon}{r+S}\left[\big\lceil\frac{B(r+S)}{\epsilon}\big\rceil\right], \nonumber \\
& \left\|\bm W\right\|_0 \le S,  \left|b_j\right| \in \frac{\epsilon}{r+S}\left[\big\lceil\frac{B(r+S)}{\epsilon}\big\rceil\right], \left|a_j\right| = 1\biggr\},\nonumber 
\end{align}
where $[n] \defn \left\{1, 2, \ldots, n-1\right\}$.
Then, for all $\left(\bm W, \bm a, \bm b\right)$ with $\|\bm W\|_1 \le B$ and $\|\bm b\|_1 \le B$, there exists some point, denoted by $\left(\widetilde{\bm W}, \widetilde{\bm a}, \widetilde{\bm b}\right)$, in $\mathcal{N}_{\epsilon}$ whose $D_1$-distance from $\left(\bm W, \bm a, \bm b\right)$ is less than $\epsilon$.
For simplicity, define
\begin{align*}
z_i &\defn \sum_{j = 1}^r a_j\mathrm{relu}(\bm w_j^{\top}\bm x_i + b_j) - \sum_{j = 1}^r a_j^{\star}\mathrm{relu}(\bm w_j^{\star\top}\bm x_i + b_j^{\star}), \\
\widetilde{z}_i &\defn \sum_{j = 1}^r \widetilde{a}_j\mathrm{relu}(\widetilde{\bm w}_j^{\top}\bm x_i + \widetilde{b}_j) - \sum_{j = 1}^r a_j^{\star}\mathrm{relu}(\bm w_j^{\star\top}\bm x_i + b_j^{\star}).
\end{align*}
Similar to~\eqref{eq:crude-bound}, we can derive that
\begin{align*}
&\left|\sum_{j = 1}^r a_j\mathrm{relu}(\bm w_j^{\top}\bm x + b_j) - \sum_{j = 1}^r \widetilde{a}_j\mathrm{relu}(\widetilde{\bm w}_j^{\top}\bm x + \widetilde{b}_j)\right| \\
& \le \epsilon\max\left\{\|\bm x\|_{\infty}, 1\right\}, 
\end{align*}
which implies
\[
\left|z_i^2 - \widetilde{z}_i^2\right| \le \epsilon\left(\epsilon+D_1\right)\max\left\{\|\bm x_i\|_{\infty}^2, 1\right\},
\]
and then with probability at least $1-\delta$,
\begin{align}
&\frac{1}{n} \sum_{i = 1}^n \left(z_i^2 - \mathbb{E}\left[z_i^2\right]\right) - \frac{1}{n} \sum_{i = 1}^n \left(\widetilde{z}_i^2 - \mathbb{E}\left[\widetilde{z}_i^2\right]\right)  \nonumber \\
&\ge - 4\epsilon\left(\epsilon+D_1\right)\log\frac{4pn}{\delta}. \label{eq:net-approx}
\end{align}

In addition, it can be verified that
\begin{align}
\log\left|\mathcal{N}_{\epsilon}\right| \le C_5S\log\frac{pr}{S}\log\frac{BS}{\epsilon}, \label{eq:net-size}
\end{align}
for some universal constant $C_5 > 0$.
Combining \eqref{eq:random-error-fixed}, \eqref{eq:net-approx}, and \eqref{eq:net-size} leads to
\begin{align*}
\frac{1}{n} \sum_{i = 1}^n \left(z_i^2 - \mathbb{E}\left[z_i^2\right]\right) \ge 
&-5\left(\epsilon+D_1\right)^2\log\frac{4pn}{\delta}\sqrt{\frac{1}{n}\log\frac{2\left|\mathcal{N}_{\epsilon}\right|}{\delta}} \\
&-4\epsilon\left(\epsilon+D_1\right)\log\frac{4pn}{\delta}.
\end{align*}
It follows that \eqref{eq:empirical} holds.

\subsection{Proof of \eqref{eq:population}}

We first consider a simple case that $b_j = 0$ and $b_j^{\star} = 0$ for $1 \le j \le r$, and show that for some small constant $C_6 > 0$,
\begin{align}
&\mathbb{E}\left[\sum_{j = 1}^r a_j\mathrm{relu}(\bm w_j^{\top}\bm x) - \sum_{j = 1}^r a_j^{\star}\mathrm{relu}(\bm w_j^{\star\top}\bm x)\right]^2 \nonumber \\
&\ge C_6\min\left\{\frac{1}{r}, D_2^2\right\}. \label{eq:E-bound}
\end{align}
\jie{Next, we will assume that
\[
\mathbb{E}\left[\sum_{j = 1}^r a_j\mathrm{relu}(\bm w_j^{\top}\bm x) - \sum_{j = 1}^r a_j^{\star}\mathrm{relu}(\bm w_j^{\star\top}\bm x)\right]^2 \le \frac{C_6}{r}.
\]
Otherwise, Inequality~\ref{eq:E-bound} already holds.}
According to Lemma~\ref{lem:E-tensor}, one has for any constant $k \ge 0$, there exists some constant $\alpha_k > 0$ such that
\begin{align}
&\mathbb{E}\left[\sum_{j = 1}^r a_j\mathrm{relu}(\bm w_j^{\top}\bm x) - \sum_{j = 1}^r a_j^{\star}\mathrm{relu}(\bm w_j^{\star\top}\bm x)\right]^2 \nonumber\\
&\ge \alpha_k\biggl\|\sum_{j = 1}^r a_j\|\bm w_j\|_2\big(\frac{\bm w_j}{\|\bm w_j\|_2}\big)^{\otimes 2k} \nonumber\\
&\hspace{1cm} - \sum_{j = 1}^r a_j^{\star}\|\bm w_j^{\star}\|_2\big(\frac{\bm w_j^{\star}}{\|\bm w_j^{\star}\|_2}\big)^{\otimes 2k}\biggr\|_F^2. \label{eq:E-tensor-nobias}
\end{align}

Assumption~\ref{assumption2} tells us that for any integer $k \ge \frac{2}{\omega}$,
\begin{align}
\left|\langle \bm v_{j_1}^{\star}, \bm v_{j_2}^{\star} \rangle\right| \le \frac{1}{r^2}. \label{eq:perp-v}
\end{align}
where
\[
\bm v_j \defn \mathrm{vec}\left(\big(\frac{\bm w_j}{\|\bm w_j\|_2}\big)^{\otimes k}\right) \quad\text{with}\quad \beta_j \defn a_j\|\bm w_j\|_2,
\]
and
\[
\bm v_j^{\star} \defn \mathrm{vec}\left(\big(\frac{\bm w_j^{\star}}{\|\bm w_j^{\star}\|_2}\big)^{\otimes k}\right)\quad\text{with}\quad \beta_j^{\star} \defn a_j^{\star}\|\bm w_j^{\star}\|_2.
\]
Then, \eqref{eq:E-tensor-nobias} gives
\begin{align*}
&\mathbb{E}\left[\sum_{j = 1}^r a_j\mathrm{relu}(\bm w_j^{\top}\bm x) - \sum_{j = 1}^r a_j^{\star}\mathrm{relu}(\bm w_j^{\star\top}\bm x)\right]^2 \\
& \ge \alpha_{3k}\left\|\sum_{j = 1}^r \beta_j\bm v_j^{\otimes 6} - \sum_{j = 1}^r \beta_j^{\star}\bm v_j^{\star\otimes 6}\right\|_F^2.
\end{align*}
Define 
\begin{align*}
\mathbb{S}_{+} \defn \mathrm{span}\left\{\bm v_j\right\}_{j : \beta_j > 0}\quad\text{ }\quad \mathbb{S}_{-} \defn \mathrm{span}\left\{\bm v_j\right\}_{j : \beta_j < 0},
\end{align*}
and
\begin{align*}
\mathbb{S}_{+}^{\star} \defn \mathrm{span}\left\{\bm v_j^{\star}\right\}_{j : \beta_j^{\star} > 0}\quad\text{ }\quad \mathbb{S}_{-}^{\star} \defn \mathrm{span}\left\{\bm v_j^{\star}\right\}_{j : \beta_j^{\star} < 0}.
\end{align*}
Let $\bm P_{\mathbb{S}}$ and $\bm P_{\mathbb{S}}^{\perp}$ denote the projection onto and perpendicular to the subspace $\mathbb{S}$, respectively. 
By noticing that $\bm P_{\mathbb{S}_{-}}^{\perp}\bm v_j = \bm 0$ for $j$ obeying $\beta_j < 0$, and $\bm P_{\mathbb{S}_{+}^{\star}}^{\perp}\bm v_j^{\star} = \bm 0$ for $j$ obeying $\beta_j^{\star} > 0$, one has
\begin{align*}
&\left\|\sum_{j = 1}^r \beta_j\bm v_j^{\otimes 6} - \sum_{j = 1}^r \beta_j^{\star}\bm v_j^{\star\otimes 6}\right\|_F^2 \\
&\ge \biggl\|\sum_{j : \beta_j > 0} \beta_j\big(\bm P_{\mathbb{S}_{-}}^{\perp}\bm v_j\big)^{\otimes 2} \otimes \big(\bm P_{\mathbb{S}_{+}^{\star}}^{\perp}\bm v_j\big)^{\otimes 4} \\
&\hspace{1cm} - \sum_{j : \beta_j^{\star} < 0} \beta_j^{\star}\big(\bm P_{\mathbb{S}_{-}}^{\perp}\bm v_j^{\star}\big)^{\otimes 2} \otimes \big(\bm P_{\mathbb{S}_{+}^{\star}}^{\perp}\bm v_j^{\star}\big)^{\otimes 4}\biggr\|_F^2 \\
&\ge \sum_{j : \beta_j^{\star} < 0} \left\|\beta_j^{\star}\big(\bm P_{\mathbb{S}_{-}}^{\perp}\bm v_j^{\star}\big)^{\otimes 2} \otimes \big(\bm P_{\mathbb{S}_{+}^{\star}}^{\perp}\bm v_j^{\star}\big)^{\otimes 4}\right\|_F^2 \\
&\ge \frac{1}{2}\sum_{j : \beta_j^{\star} < 0} \left\|\bm P_{\mathbb{S}_{-}}^{\perp}\bm v_j^{\star}\right\|_2^4,
\end{align*}
where the penultimate inequality holds since the inner product between every pair of terms is positive, and the last inequality comes from the facts that $|\beta_j^{\star}| \ge 1$ and \eqref{eq:perp-v}.

Moreover, by means of AM-GM inequality and \eqref{eq:perp-v}, one can see that
\begin{align*}
\sum_{j : \beta_j^{\star} < 0} \left\|\bm P_{\mathbb{S}_{-}}^{\perp}\bm v_j^{\star}\right\|_2^4 
&\ge \frac{1}{r}\Big(\sum_{j : \beta_j^{\star} < 0} \left\|\bm P_{\mathbb{S}_{-}}^{\perp}\bm v_j^{\star}\right\|_2^2\Big)^2 \\
&= \frac{1}{r}\left\|\bm P_{\mathbb{S}_{-}}^{\perp} \big[\bm v_j^{\star}\big]_{j : \beta_j^{\star} < 0}\right\|_F^4 \\
&\ge \frac{1}{2r}\left\|\bm P_{\mathbb{S}_{-}}^{\perp} \bm P_{\mathbb{S}_{-}^{\star}}\right\|_F^4.
\end{align*}
Then combining with \eqref{eq:E-bound}, the above result and the counterpart for $\beta_j^{\star} > 0$ lead to
\begin{align*}
\mathrm{dim}(\mathbb{S}_{-}) \ge \mathrm{dim}(\mathbb{S}_{-}^{\star})\quad\text{and}\quad \mathrm{dim}(\mathbb{S}_{+}) \ge \mathrm{dim}(\mathbb{S}_{+}^{\star}),
\end{align*}
which gives
\begin{align*}
\mathrm{dim}(\mathbb{S}_{-}) = \mathrm{dim}(\mathbb{S}_{-}^{\star})\quad\text{and}\quad \mathrm{dim}(\mathbb{S}_{+}) = \mathrm{dim}(\mathbb{S}_{+}^{\star}).
\end{align*}
Furthermore, for some small constant $C_6 > 0$, we have
\begin{align*}
\mathrm{dist}(\mathbb{S}_{-}, \mathbb{S}_{-}^{\star}) \le C_6 \quad\text{and}\quad \mathrm{dist}(\mathbb{S}_{+}, \mathbb{S}_{+}^{\star}) \le C_6.
\end{align*}

Let $\bm P_{i}^{\perp}$ denote the projection perpendicular to 
\[
\mathrm{span}\left\{\bm v_j^{\star}\right\}_{j \ne i : \beta_j^{\star} > 0},
\]
and
\[
\gamma_j \defn \frac{\beta_j\langle \bm P_{\mathbb{S}_{-}}^{\perp}\bm v_j, P_{\mathbb{S}_{-}}^{\perp}\bm v_i^{\star}\rangle^2\langle \bm P_{i}^{\perp}\bm v_i, P_{\mathbb{S}_{-}}^{\perp}\bm v_i^{\star}\rangle^2}{\big\|\bm P_{\mathbb{S}_{-}}^{\perp}\bm v_i^{\star}\big\|_2^{2}\big\|\bm P_{i}^{\perp}\bm v_i^{\star}\big\|^{2}}.
\]
Then for any $i$,
\begin{align*}
&\left\|\sum_{j = 1}^r \beta_j\bm v_j^{\otimes 6} - \sum_{j = 1}^r \beta_j^{\star}\bm v_j^{\star\otimes 6}\right\|_F^2 \\
&\ge \biggl\|\sum_{j : \beta_j > 0} \beta_j\big(\bm P_{\mathbb{S}_{-}}^{\perp}\bm v_j\big)^{\otimes 2} \otimes \bm v_j^{\otimes 4} \\
&\hspace{1cm} - \sum_{j = 1}^r \beta_j^{\star}\big(\bm P_{\mathbb{S}_{-}}^{\perp}\bm v_j^{\star}\big)^{\otimes 2} \otimes \bm v_j^{\star\otimes 4}\biggr\|_F^2 \\
&\ge \frac{1}{2}\biggl\|\sum_{j : \beta_j > 0} \beta_j\big(\bm P_{\mathbb{S}_{-}}^{\perp}\bm v_j\big)^{\otimes 2} \otimes \bm v_j^{\otimes 4} \\
&\hspace{1cm} - \sum_{j : \beta_j^{\star} > 0} \beta_j^{\star}\big(\bm P_{\mathbb{S}_{-}}^{\perp}\bm v_j^{\star}\big)^{\otimes 2} \otimes \bm v_j^{\star\otimes 4}\biggr\|_F^2 \\
&\ge \frac{1}{2}\biggl\|\sum_{j : \beta_j > 0} \beta_j\big(\bm P_{\mathbb{S}_{-}}^{\perp}\bm v_j\big)^{\otimes 2} \otimes \big(\bm P_{i}^{\perp}\bm v_i\big)^{\otimes 2} \otimes \bm v_j^{\otimes 2} \\
&\hspace{1cm}- \beta_i^{\star}\big(\bm P_{\mathbb{S}_{-}}^{\perp}\bm v_i^{\star}\big)^{\otimes 2} \otimes \big(\bm P_{i}^{\perp}\bm v_i^{\star}\big)^{\otimes 2} \otimes \bm v_i^{\star\otimes 2}\biggr\|_F^2 \\
&\ge \frac{1}{2}\left\|\sum_{j : \beta_j > 0} \gamma_j\bm v_j^{\otimes 2} - \beta_i^{\star}\big\|\bm P_{\mathbb{S}_{-}}^{\perp}\bm v_i^{\star}\big\|_2^{2}\big\|\bm P_{i}^{\perp}\bm v_i^{\star}\big\|^{2} \bm v_i^{\star\otimes 2}\right\|_F^2,
\end{align*}
which, together with \eqref{eq:E-bound}, implies that there exists some $j$ such that
\begin{align*}
\|\sqrt{\beta_j}\bm v_j - \sqrt{\beta_i^{\star}}\bm v_i^{\star}\|_2^2 \le \frac{1}{r}.
\end{align*}

Without loss of generality, assume that
\begin{align}
\|\sqrt{\beta_j}\bm v_j - \sqrt{\beta_j^{\star}}\bm v_j^{\star}\|_2^2 \le \frac{1}{r}, \qquad \text{for all }1 \le j \le r.
\end{align}
Then
\begin{align*}
&\mathbb{E}\big[\sum_{j = 1}^r a_j\mathrm{relu}(\bm w_j^{\top}\bm x) - \sum_{j = 1}^r a_j^{\star}\mathrm{relu}(\bm w_j^{\star\top}\bm x)\big]^2 \\
&\ge \alpha_k\left\|\sum_{j = 1}^r \beta_j\bm v_j\bm v_j^{\top} - \sum_{j = 1}^r \beta_j^{\star}\bm v_j^{\star}\bm v_j^{\star\top}\right\|_F^2 \\
&\ge \alpha_k\sum_{j = 1}^r \left\|\beta_j\bm v_j\bm v_j^{\top} - \beta_j^{\star}\bm v_j^{\star}\bm v_j^{\star\top}\right\|_F^2 \\
&\quad - \frac{\alpha_k}{2r}\left(\sum_{j = 1}^r \left\|\beta_j\bm v_j\bm v_j^{\top} - \beta_j^{\star}\bm v_j^{\star}\bm v_j^{\star\top}\right\|_F\right)^2 \\
&\ge \frac{\alpha_k}{2}\sum_{j = 1}^r \left\|\beta_j\bm v_j\bm v_j^{\top} - \beta_j^{\star}\bm v_j^{\star}\bm v_j^{\star\top}\right\|_F^2.
\end{align*}
Here, the first line comes from \eqref{eq:E-tensor-nobias};
the second line holds through the following claim 
\begin{align*}
&\left|\langle \beta_{j_1}\bm v_{j_1}\bm v_{j_1}^{\top} - \beta_{j_1}^{\star}\bm v_{j_1}^{\star}\bm v_{j_1}^{\star\top}, \beta_{j_2}\bm v_{j_2}\bm v_{j_2}^{\top} - \beta_{j_2}^{\star}\bm v_{j_2}^{\star}\bm v_{j_2}^{\star\top}\rangle\right| \\
&\le \frac{1}{2r}\|\beta_{j_1}\bm v_{j_1}\bm v_{j_1}^{\top} - \beta_{j_1}^{\star}\bm v_{j_1}^{\star}\bm v_{j_1}^{\star\top}\|_2\|\beta_{j_2}\bm v_{j_2}\bm v_{j_2}^{\top} - \beta_{j_2}^{\star}\bm v_{j_2}^{\star}\bm v_{j_2}^{\star\top}\|_2
\end{align*}
since for $\bm \delta_j \defn \sqrt{\beta_j}\bm v_j - \sqrt{\beta_j^{\star}}\bm v_j^{\star}$,
\begin{align*}
\beta_j\bm v_j\bm v_j^{\top} - \beta_j^{\star}\bm v_j^{\star}\bm v_j^{\star\top} = \bm \delta_j\bm \delta_j^{\top} + \sqrt{\beta_j^{\star}}\bm \delta_j\bm v_j^{\star\top} + \sqrt{\beta_j^{\star}}\bm v_j^{\star}\bm \delta_j^{\star\top}.
\end{align*}
Then the conclusion is obvious by noticing that
\begin{align*}
\left\|\beta_j\bm v_j\bm v_j^{\top} - \beta_j^{\star}\bm v_j^{\star}\bm v_j^{\star\top}\right\|_F \ge \|\bm w_j - \bm w_j^{\star}\|_2.
\end{align*}

Finally, we analyze the general case with $b_j, b_j^{\star} \ne 0$, which is similar to the above argument.
For simplicity, we only explain the different parts here.
According to Lemma~\ref{lem:E-tensor}, one has for any constant $k \ge 0$, there exists some constant $\alpha_k > 0$ and some function $f_k : \mathbb{R} \rightarrow \mathbb{R}$ such that
\begin{align}
&\mathbb{E}\left[\sum_{j = 1}^r a_j\mathrm{relu}(\bm w_j^{\top}\bm x + b_j) - \sum_{j = 1}^r a_j^{\star}\mathrm{relu}(\bm w_j^{\star\top}\bm x + b_j^{\star})\right]^2 \nonumber\\
&\ge \sum_{k \ge \frac{12}{\omega}}^{\infty}\biggl\|\sum_{j = 1}^r a_jf_k(\frac{b_j}{\|\bm w_j\|_2})\|\bm w_j\|_2\big(\frac{\bm w_j}{\|\bm w_j\|_2}\big)^{\otimes k} \nonumber  \\
&\hspace{1.5cm}- \sum_{j = 1}^r a_j^{\star}f_k(\frac{b_j^{\star}}{\|\bm w_j^{\star}\|_2})\|\bm w_j^{\star}\|_2\big(\frac{\bm w_j^{\star}}{\|\bm w_j^{\star}\|_2}\big)^{\otimes k}\biggr\|_F^2 \nonumber\\
&\gtrsim \sum_{j = 1}^r \sum_{k \ge \frac{12}{\omega}}^{\infty}\left\|a_jf_k(\frac{b_j}{\|\bm w_j\|_2})\bm w_j - a_j^{\star}f_k(\frac{b_j^{\star}}{\|\bm w_j^{\star}\|_2})\bm w_j^{\star}\right\|_F^2 \nonumber\\
&\gtrsim \sum_{j = 1}^r \inf_{R_l(\bm x)} \mathbb{E}\biggl[a_j\mathrm{relu}(\bm w_j^{\top}\bm x + b_j) \\
&\hspace{2.2cm} - a_j^{\star}\mathrm{relu}(\bm w_j^{\star\top}\bm x + b_j^{\star}) - R_l(\bm x)\biggr]^2 \nonumber\\
&\gtrsim \sum_{j = 1}^r \left(\|\bm w_j - \bm w_j^{\star}\|_2^2 + |b_j - b_j^{\star}|^2\right) \nonumber. 
\end{align}
Here, $l = \left[\frac{12}{\omega}\right]$, and the second inequality holds in a similar way to above analysis. 
Then the general conclusion is handy.

\section{Proof of Lemma~\ref{lem:upper} (upper bound)}

For simplicity, let
\begin{align*}
z_i &\defn \sum_{j = 1}^r a_j\mathrm{relu}(\bm w_j^{\top}\bm x_i + b_j) - \sum_{j = 1}^r a_j^{\star}\mathrm{relu}(\bm w_j^{\star\top}\bm x_i + b_j^{\star}), \\
\widehat{z}_i &\defn \sum_{j = 1}^r a_j\mathrm{relu}(\bm w_j^{\top}\bm x_i + b_j) - \sum_{j = 1}^r \widehat{a}_j\mathrm{relu}(\widehat{\bm w}_j^{\top}\bm x_i + \widehat{b}_j).
\end{align*}
Recall the optimality of $\left(\widehat{\bm W}, \widehat{\bm a}, \widehat{\bm b}\right)$ with respect to the problem in~\ref{eq:l1}.
According to the triangle inequality, one has
\begin{align}
\sqrt{\frac{1}{n} \sum_{i = 1}^n z_i^2} \le \sqrt{\frac{1}{n} \sum_{i = 1}^n \widehat{z}_i^2} + 2\sigma.
\end{align}
We can bound the first term in the right hand side by
\begin{align*}
&\frac{1}{n} \sum_{i = 1}^n \widehat{z}_i^2\\
&= \frac{1}{n}\sum_{i = 1}^n \left[\sum_{j = 1}^r a_j\left(\mathrm{relu}(\bm w_j^{\top}\bm x_i + b_j) - \mathrm{relu}(\widehat{\bm w}_j^{\top}\bm x_i + \widehat{b}_j)\right)\right]^2 \\
&\leq \frac{1}{n}\sum_{i = 1}^n \left[\sum_{j = 1}^r \left|(\bm w_j-\widehat{\bm w}_j)^{\top}\bm x_i\right|\right]^2 \\
&\le \frac{r}{n}\sum_{i = 1}^n \sum_{j = 1}^r \left|(\bm w_j-\widehat{\bm w}_j)^{\top}\bm x_i\right|^2,
\end{align*}
where the second line holds due to the contraction property of ReLu function,
and the last line comes from the AM-GM inequality. 
Lemma~\ref{lem:linear} further gives for some constant $C_7 > 0$,
\begin{align*}
\sum_{j = 1}^r \frac{1}{n}\sum_{i = 1}^n \left|(\bm w_j-\widehat{\bm w}_j)^{\top}\bm x_i\right|^2 \le &C_7\sum_{j = 1}^r \left\|\bm w_j-\widehat{\bm w}_j\right\|_2^2 \\
&+ C_7\frac{\log^3 \frac{p}{n\delta}}{n} \sum_{j = 1}^r \left\|\bm w_j-\widehat{\bm w}_j\right\|_1^2
\end{align*}
holds with probability at least $1 - \delta$.
In addition,
\begin{align*}
\sum_{j = 1}^r \left\|\bm w_j-\widehat{\bm w}_j\right\|_1^2 
&\le \left\|\bm W-\widehat{\bm W}\right\|_1^2 \\
& \le \left(\|\bm W^{\star}\|_1-\|\widehat{\bm W}\|_1\right)^2 \le D_1^2,
\end{align*}
and
\begin{align*}
\sum_{j = 1}^r \left\|\bm w_j-\widehat{\bm w}_j\right\|_2^2 &= \left\|\bm W-\widehat{\bm W}\right\|_1^2 \le \left\|\bm W-\widehat{\bm W}\right\|_1\left\|\bm W-\widehat{\bm W}\right\|_{\infty} \\
&\le \frac{\left(\|\bm W^{\star}\|_1-\|\widehat{\bm W}\|_1\right)\left(\|\bm W^{\star}\|_1-\|\widehat{\bm W}^{\star}\|_1\right)}{S/2} \\
&\le \frac{4}{S}D_1^2.
\end{align*}
Here, $\widehat{\bm W}^{\star}$ denote the entries of $\widehat{\bm W}$ on the support set for $\bm W^{\star}$, and we make use of the fact that $\|\widehat{\bm W}\|_1 \le \|\bm W^{\star}\|_1$ and
\begin{align*}
\left\|\bm W-\widehat{\bm W}\right\|_{\infty} \le \frac{\|\widehat{\bm W}^{\star}-\widehat{\bm W}\|_1}{S-s} \le \frac{\|\bm W^{\star}\|_1-\|\widehat{\bm W}^{\star}\|_1}{S/2}.
\end{align*}
Putting everything together gives the desired result.

\section{Technical lemmas} \label{sec_tech}


\begin{lemma} \label{lem:D1-D2}
For any $\left(\bm W, \bm a, \bm b\right)$ with $\left\|\bm W\right\|_0 + \left\|\bm b\right\|_0 + \left\|\bm W^{\star}\right\|_0 + \left\|\bm b^{\star}\right\|_0 \le S$. Assume that $\left\|\bm W\right\|_1 + \left\|\bm b\right\|_1 \le \left\|\bm W^{\star}\right\|_1 + \left\|\bm b^{\star}\right\|_1$ and $\|\bm w_j^{\star}\|_2^2 + |b_j^{\star}|^2 \le 1$.
Then one has
\begin{align}
D_1 \le 2\sqrt{S}D_2,
\end{align}
where $D_1, D_2$ are defined in~\eqref{eq:distance}.
\end{lemma}

\begin{proof}
For simplicity, assume that 
\[
D_2^2 = \sum_{j\in\mathcal{J}} \left(\|\bm w_j - \bm w_j^{\star}\|_2^2 + |b_j - b_j^{\star}|^2\right) + \sum_{j\notin\mathcal{J}} \left(\|\bm w_j^{\star}\|_2^2 + |b_j^{\star}|^2\right).
\]
Here, $j\in\mathcal{J}$ means that $a_j = a_j^{\star}$ and 
\[
\|\bm w_j - \bm w_j^{\star}\|_2^2 + |b_j - b_j^{\star}|^2 \le \|\bm w_j^{\star}\|_2^2 + |b_j^{\star}|^2.
\]
Then, according to the AM-GM inequality, one has
\begin{align*}
\sqrt{S}D_2 &\ge \sum_{j\in\mathcal{J}} \biggl(\|\bm w_j - \bm w_j^{\star}\|_1 + |b_j - b_j^{\star}|\biggr) \\
& \quad + \sum_{j\notin\mathcal{J}} \biggl(\|\bm w_j^{\star}\|_1 + |b_j^{\star}|\biggr) \\
&\ge \sum_{j\in\mathcal{J}} \biggl(\|\bm w_j^{\star}\|_1 - \|\bm w_j\|_1 + |b_j^{\star}| - |b_j|\biggr) + \|\bm W^{\star}\|_1 \\
&\quad + \|\bm b^{\star}\|_1 - \sum_{j\in\mathcal{J}} \biggl(\|\bm w_j^{\star}\|_1 + |b_j^{\star}|\biggr) \\
&\ge \sum_{j\notin\mathcal{J}} \biggl(\|\bm w_j\|_1 + |b_j|\biggr),
\end{align*}
which implies that
\begin{align*}
2\sqrt{S}D_2 \ge & \sum_{j\in\mathcal{J}} \left(\|\bm w_j - \bm w_j^{\star}\|_1 + |b_j - b_j^{\star}|\right) \\
&+ \sum_{j\notin\mathcal{J}} \left(\|\bm w_j^{\star}\|_1 + |b_j^{\star}| + \|\bm w_j\|_1 + |b_j|\right).
\end{align*}
Thus we conclude the proof.
\end{proof}

\jie{
\begin{lemma}[Theorem 2.1~\cite{ge2017learning}] \label{lem:E-tensor}
For any constant $k \ge 0$, there exists some universal function $f_k : \mathbb{R} \rightarrow \mathbb{R}$ such that
\begin{align}
&\mathbb{E}\left[\sum_{j = 1}^r a_j\mathrm{relu}(\bm w_j^{\top}\bm x + b_j) - \sum_{j = 1}^r a_j^{\star}\mathrm{relu}(\bm w_j^{\star\top}\bm x + b_j^{\star})\right]^2 \nonumber\\
&= \sum_{k = 0}^{\infty}\biggl\|\sum_{j = 1}^r a_jf_k\biggl(\frac{b_j}{\|\bm w_j\|_2}\biggr)\|\bm w_j\|_2\biggl(\frac{\bm w_j}{\|\bm w_j\|_2}\biggr)^{\otimes k} \nonumber\\
&\qquad \quad - \sum_{j = 1}^r a_j^{\star}f_k\biggl(\frac{b_j^{\star}}{\|\bm w_j^{\star}\|_2}\biggr)\|\bm w_j^{\star}\|_2\biggl(\frac{\bm w_j^{\star}}{\|\bm w_j^{\star}\|_2}\biggr)^{\otimes k}\biggr\|_F^2, \label{eq:E-tensor-general}
\end{align}
with
\begin{align}
\alpha_{k} \defn f_{2k}(0) > 0,\qquad\text{for all }k > 0.
\end{align}
In addition, we have
\begin{align}
&\inf_{R_l} \mathbb{E}\left[a\mathrm{relu}(\bm w^{\top}\bm x + b) - \sum_{j = 1}^r a^{\star}\mathrm{relu}(\bm w^{\star\top}\bm x + b^{\star}) - R_l(\bm x)\right]^2 \nonumber\\
&= \sum_{k > l}^{\infty}\biggl\|a f_k\biggl(\frac{b}{\|\bm w\|_2}\biggr)\|\bm w\|_2\biggl(\frac{\bm w}{\|\bm w\|_2}\biggr)^{\otimes k} \nonumber \\
&\hspace{1cm}- a^{\star}f_k\biggl(\frac{b^{\star}}{\|\bm w^{\star}\|_2}\biggr)\|\bm w^{\star}\|_2\biggl(\frac{\bm w^{\star}}{\|\bm w^{\star}\|_2}\biggr)^{\otimes k}\biggr\|_F^2,
\end{align}
where $R_l$ denote a polynomial function of $\bm x$ with degree less than $l$.
\end{lemma}
}

\begin{lemma} \label{lem:linear}
There exists some universal constant $c > 0$, such that for all $\bm w \in \mathbb{R}^p$,
\begin{align}
\frac{1}{n}\sum_{i = 1}^n \left|\bm w^{\top}\bm x_i\right|^2 \le c\left\|\bm w\right\|_2^2 + c\frac{\log^3 \frac{p}{n\delta}}{n} \left\|\bm w\right\|_1^2,
\end{align}
holds with probability at least $1 - \delta$.
\end{lemma}

\begin{proof}
Before proceeding, we introduce some useful techniques about Restricted Isometry Property (RIP). 
Let $\bm X \defn \frac{1}{\sqrt{n}}[\bm x_1, \bm x_2, \ldots, \bm x_n]$.
For some constant $c_0 > 0$, if $n \ge c_0\left(s\log \frac{p}{s} + \log\frac{1}{\delta}\right)$, then with probability at least $1-\delta$,
\begin{align}
\left\|\bm X^{\top}\bm w\right\|_2^2 \le 2\|\bm w\|_2^2 \label{eq:RIP}
\end{align}
holds for all $\bm w$ satisfying $\|\bm w\|_0 \le s$.

We divide the entries of $\bm w$ into several groups $\mathcal{S}_1 \cup \mathcal{S}_2 \cup \ldots \cup \mathcal{S}_L$ with equal size $s$ (except for $\mathcal{S}_L$), such that the entries in $\mathcal{S}_j$ are no less than $\mathcal{S}_k$ for any $j < k$.
Then, according~\eqref{eq:RIP}, one has
\begin{align*}
\frac{1}{n}\sum_{i = 1}^n(\bm w^{\top}\bm x_i)^2 
&= \bm w^{\top}\bm X\bm X^{\top}\bm w = \sum_{j, k}\bm w_{\mathcal{S}_j}^{\top}\bm X\bm X^{\top}\bm w_{\mathcal{S}_k} \\
&\le 2\sum_{j, k}\|\bm w_{\mathcal{S}_j}\|_2\|\bm w_{\mathcal{S}_k}\|_2 
= 2\Big(\sum_{l = 1}^L\|\bm w_{\mathcal{S}_l}\|_2\Big)^2.
\end{align*}
In addition, the order of $\bm w_{\mathcal{S}_l}$ yields for $l > 1$, 
\begin{align*}
\|\bm w_{\mathcal{S}_l}\|_2 \le \sqrt{s}\|\bm w_{\mathcal{S}_l}\|_{\infty} \le \frac{1}{(l-1)\sqrt{s}}\|\bm w\|_1,
\end{align*}
which leads to
\begin{align*}
\Big(\sum_{l = 1}^L\|\bm w_{\mathcal{S}_l}\|_2\Big)^2 
&\le 2\|\bm w_{\mathcal{S}_1}\|_2^2 + 2\Big(\sum_{l = 2}^L\frac{1}{(l-1)\sqrt{s}}\|\bm w\|_1\Big)^2 \\
&\le 2\|\bm w\|_2^2 + \frac{2\log^2 L}{s}\|\bm w\|_1^2.
\end{align*}
We conclude the proof by combining the above inequalities.
\end{proof}

\section{Further experiments details}\label{appendix_exp}

    The hyper-parameters used in \Autoref{sec_exp} are summarized in \Autoref{tab:hyper}. 
    \begin{table*}[tb]
    \centering
    \caption{Hyper-parameters used in our experiments.}
    \label{tab:hyper}
    \begin{tabular}{@{}ccccccc@{}}
    \toprule
    \multicolumn{2}{c}{Dataset} & Linear & NN-generated & Friedman & BGSBoy & UJIINdoorLoc \\ \midrule
    \multicolumn{2}{c}{Epochs} & \multicolumn{5}{c}{\{100,200,500\}} \\ \cmidrule(l){3-7} 
    \multicolumn{2}{c}{Batch size} & 32 & 32 & 32 & 8 & 128 \\ \cmidrule(l){3-7} 
      \multicolumn{2}{c}{Optimizer} & \multicolumn{5}{c}{ADAM}  \\ \cmidrule(l){3-7} 
      \multicolumn{2}{c}{Learning rate} & \multicolumn{5}{c}{\{0.001, 0.005, 0.01\}}  \\ \cmidrule(l){3-7} 
     \multicolumn{2}{c}{Scheduler} & \multicolumn{5}{c}{N/A} \\ \bottomrule
    \end{tabular}
    \end{table*}

    We briefly explain the variable selection procedure. We first obtain a vector of the variables' importance. For `LASSO' and `OMP', we use the absolute value of the estimated coefficient as the variable importance; for `NN', `GLASSO', and `GSR', we obtain the importance by applying row-wise $\ell_2$-norm to the weight matrix in the input layer of the neural network; for `RF', `GB', and `LNET', we use the importance produced by those methods. 
    Once we have the importance vector, we can obtain the receiver operating characteristic (ROC) curve for synthetic datasets by varying the cut-off thresholds and calculate the AUC score. As for variable selection, we apply GMM of two mixtures to the importance vector for the synthetic datasets. The variables in the cluster with higher importance are considered significant. Then, we calculate the correctly or wrongly selected variables accordingly. 
    For BGSBoy and UJIIndoorLoc datasets, the variables with the three- and ten-largest importance are selected, respectively.

\balance
\bibliography{survey,J,NN,overview,rip_dl}
\bibliographystyle{IEEEtran}

\begin{IEEEbiographynophoto}{Gen Li} is currently a postdoctoral researcher in the Department of Statistics and Data Science at the Wharton School, University of Pennsylvania.  He received Ph.D.~in Electronic Engineering from Tsinghua University in 2021, and his B.S. degrees in Electronic Engineering and Mathematics from Tsinghua University in 2016. His research interests include reinforcement learning, high-dimensional statistics, machine learning, signal processing, and mathematical optimization. He has received the excellent graduate award and the excellent thesis award from Tsinghua University.
\end{IEEEbiographynophoto}

\begin{IEEEbiographynophoto}{Ganghua Wang} received the B.S. degree from Peking University, Beijing, China, in 2019. Since 2019, he has been a PhD student with the School of Statistics, University of Minnesota, Twin Cities, MN, USA. His research interests include the foundations of machine learning theory and machine learning safety.
\end{IEEEbiographynophoto}

\begin{IEEEbiographynophoto}{Jie Ding}
received his Ph.D. degree in Engineering Sciences from Harvard University, Cambridge, in 2017. He obtained his B.S. degree from Tsinghua University, Beijing, in 2012. He joined the faculty of the University of Minnesota, Twin Cities, in 2018, where he has been an Assistant Professor at the School of Statistics, with a graduate faculty appointment at the Department of Electrical and Computer Engineering. His research is at the intersection of machine learning, statistics, signal processing, and information theory.
\end{IEEEbiographynophoto}

\end{document}